\documentclass{article}
\usepackage[utf8]{inputenc}
\usepackage[letterpaper]{geometry}
\usepackage[parfill]{parskip}
\PassOptionsToPackage{numbers, compress}{natbib}
\usepackage{natbib}
\usepackage{xr-hyper,refcount}
\usepackage{graphicx}
\usepackage[T1]{fontenc}    
\usepackage[colorlinks, citecolor=blue, urlcolor=blue, linkcolor=blue]{hyperref}       
\usepackage{url}            
\usepackage{booktabs}       
\usepackage{amsfonts}       
\usepackage{nicefrac}       
\usepackage{microtype}      
\usepackage[parfill]{parskip}
\usepackage{algorithm,algorithmic}
\usepackage{amsmath,amsthm,amssymb,bbm}
\usepackage{mathtools}
\usepackage{cases}
\usepackage{comment}
\usepackage{subcaption}
\usepackage{algorithm,algorithmic}
\usepackage{color}
\usepackage{xcolor}
\usepackage{appendix}
\usepackage{xspace}
\usepackage[shortlabels]{enumitem}
\usepackage[english]{babel}
\usepackage{authblk}
\usepackage{libertine}
\usepackage{libertinust1math}
\usepackage{bm}
\usepackage{arydshln}
\usepackage{makecell}
\usepackage{multirow}
\usepackage{framed}
\usepackage[font=small]{caption}
\usepackage[scientific-notation=true]{siunitx}
\usepackage{wrapfig}
\usepackage{lipsum}
\usepackage{sidecap}
\usepackage{pifont}
\usepackage{tikz}
\theoremstyle{plain}

\newcommand{\fed}{\textsc{FedAvg}}
\newcommand{\fedprox}{\textsc{FedProx}}
\newcommand{\ours}{\textsc{LG-FedAvg}}
\newcommand{\ourl}{Local Global Federated Averaging}

\DeclareMathOperator*{\argmax}{arg\,max}
\DeclareMathOperator*{\argmin}{arg\,min}

\definecolor{gg}{RGB}{15,150,15}
\definecolor{rr}{RGB}{190,45,45}

\makeatletter
\def\maketag@@@#1{\hbox{\m@th\normalfont\normalsize#1}}
\makeatother

\usepackage{amsmath,amsfonts,bm}

\def\eqref#1{equation~\ref{#1}}

\def\1{\bm{1}}

\DeclareMathAlphabet{\mathsfit}{\encodingdefault}{\sfdefault}{m}{sl}
\SetMathAlphabet{\mathsfit}{bold}{\encodingdefault}{\sfdefault}{bx}{n}

\usepackage{hyperref}

\usepackage{booktabs}       %
\usepackage{amsfonts}       %
\usepackage{nicefrac}       %
\usepackage{microtype}      %
\usepackage{subcaption}

\usepackage{enumitem}
\setlist{nolistsep}
\setlist[itemize]{noitemsep, topsep=0pt}

\usepackage{times}

\usepackage{graphicx} %
\usepackage{color}

\usepackage{array}
\usepackage{amssymb}
\usepackage{amsmath}
\usepackage{xspace}
\usepackage{fancyhdr}
\usepackage{comment}

\newcolumntype{H}{>{\setbox0=\hbox\bgroup}c<{\egroup}@{}}

\newcommand{\noaistats}[1]{}  %

\definecolor{darkgreen}{rgb}{0,0.4,0.0}
\definecolor{darkblue}{rgb}{0,0.1,0.3}
\definecolor{darkred}{rgb}{0.7,0.0,0.0}

\newcommand{\clientfrac}{\ensuremath{C}}    %

\newcommand{\SUB}[1]{\ENSURE \hspace{-0.15in} \textbf{#1}}

\renewcommand{\algorithmicensure}{}

\newcommand{\grad}{\triangledown}

\newtheorem{theorem}{Theorem}
\newtheorem{proposition}{Proposition}

\newtheorem{corollary}{Corollary}

\newcommand\blfootnote[1]{%
  \begingroup
  \renewcommand\thefootnote{}\footnote{#1}%
  \addtocounter{footnote}{-1}%
  \endgroup
}

\makeatletter
   \def\tagform@#1{\maketag@@@{\normalsize(#1)\@@italiccorr}}
\makeatother

\AtBeginDocument{%
  \addtolength\abovedisplayskip{-0.3\baselineskip}%
  \addtolength\belowdisplayskip{-0.3\baselineskip}%
}

\title{\Large Think Locally, Act Globally:\\Federated Learning with Local and Global Representations}

\author{%
  Paul Pu Liang$^{1*}$, Terrance Liu$^{1*}$, Liu Ziyin$^2$, Nicholas B. Allen$^3$, Randy P. Auerbach$^4$, David Brent$^5$, Ruslan Salakhutdinov$^1$, Louis-Philippe Morency$^1$\\
  $^1$School of Computer Science, Carnegie Mellon University\\
  $^2$Department of Physics, University of Tokyo\\
  $^3$Department of Psychology, University of Oregon\\
  $^4$Department of Psychiatry, Columbia University\\
  $^5$Department of Psychiatry, University of Pittsburgh\\
  \texttt{\{pliang,terrancl,morency\}@cs.cmu.edu}\\
}

\begin{document}

\maketitle

\begin{abstract}
Federated learning is a method of training models on private data distributed over multiple devices. To keep device data private, the global model is trained by only communicating parameters and updates which poses scalability challenges for large models. To this end, we propose a new federated learning algorithm that jointly learns compact \textit{local representations} on each device and a global model across all devices. As a result, the global model can be smaller since it only operates on local representations, reducing the number of communicated parameters. Theoretically, we provide a generalization analysis which shows that a combination of local and global models reduces both variance in the data as well as variance across device distributions. Empirically, we demonstrate that local models enable \textit{communication-efficient} training while retaining performance. We also evaluate on the task of \textit{personalized} mood prediction from real-world mobile data where privacy is key. Finally, local models handle \textit{heterogeneous} data from new devices, and learn \textit{fair} representations that obfuscate protected attributes such as race, age, and gender.\blfootnote{$^*$first two authors contributed equally.}
\end{abstract}

\vspace{-2mm}
\section{Introduction}
\vspace{-1mm}

Federated learning is an emerging research paradigm to train machine learning models on private data distributed in a potentially non-i.i.d. setting over multiple devices~\cite{DBLP:journals/corr/McMahanMRA16}. A key challenge involves keeping private all the data on each device by training a global model only via communication of parameter updates to each device. This relies on the global model being sufficiently compact so that the parameters and updates can be sent efficiently over existing communication channels such as wireless networks~\cite{Nilsson:2018:PEF:3286490.3286559}. However, the recent demands in larger models pose a challenge for deploying federated learning on real-world tasks. In this paper, we propose a new federated learning algorithm, \ourl\ (\ours), which jointly learns compact \textit{local representations} on each device and a global model across all devices.
We perform a generalization analysis of federated learning which shows that a combination of local and global models reduces both variance in the data as well as variance across device distributions, which is more optimal than either extreme.
To support our theoretical analysis, we perform a wide range of experiments that suggest local representation learning is beneficial for the following reasons:

1) \textit{Efficiency:} Having local models extract useful, lower-dimensional representations means that the global model now requires fewer number of parameters, thereby reducing the number of parameters and updates that need to be communicated to and from the global model as well as the bottleneck in terms of communication cost. Our proposed method also maintains performance on publicly available datasets spanning image recognition (MNIST, CIFAR) and multimodal learning (VQA).


\begin{figure*}[tbp]
\centering
\includegraphics[width=0.9\linewidth]{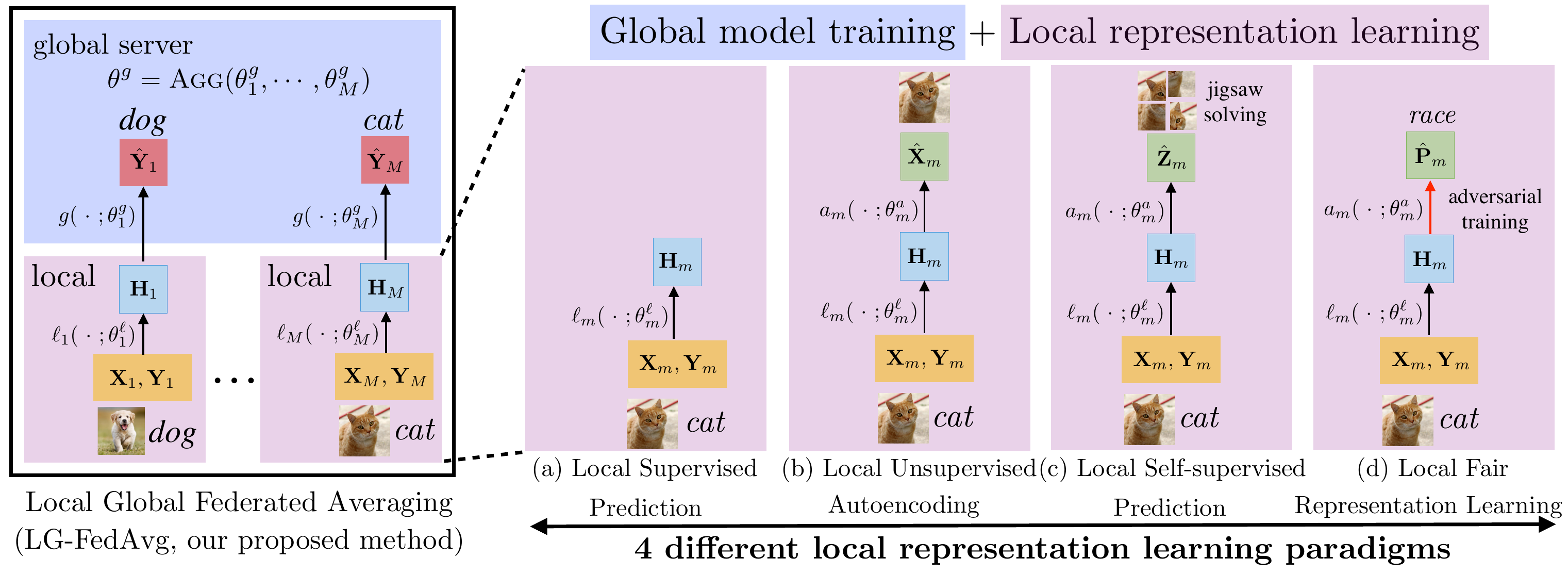}
\caption{\small (a) \ourl\ (\ours) allows for \textit{efficient} global parameter updates (smaller number of global parameters $\theta^g$), \textit{flexibility} in design across local and global models, the ability to handle \textit{heterogeneous} data, and \textit{fair} representation learning. (a) through (c) show various approaches of training local models including supervised, unsupervised, and self-supervised learning (e.g. jigsaw solving~\cite{DBLP:journals/corr/NorooziF16}). (d) shows adversarial training against protected attributes $\mathbf{P}_m$. Blue represents the global server and purple represents the local devices. $(\mathbf{X}_m, \mathbf{Y}_m)$ represents data on device $m$, $\mathbf{H}_m$ are learned local representations via local models $\ell_m(\ \cdot \ ;\theta_m^\ell) : \mathbf{x} \rightarrow \mathbf{h}$ and (optionally) auxiliary models $a_m(\ \cdot \ ;\theta_m^a) : \mathbf{h} \rightarrow \mathbf{z}$. $g(\ \cdot \ ;\theta^{g}) : \mathbf{h} \rightarrow \mathbf{y}$ is the global model. \textsc{Agg} is an aggregation function over local updates to the global model (e.g. \textsc{FedAvg}).\vspace{-4mm}}
\label{local}
\end{figure*}

2) \textit{Heterogeneity:} Real-world data is often heterogeneous (coming from different sources). A new device could contain sources of data that have never been observed before during training, such as images of a different domain or different texting styles on personalized mobile devices. Local representations allow us to process new device data using specialized encoders depending on their source modalities~\cite{DBLP:journals/corr/BaltrusaitisAM17} instead of using a single global model that might not generalize to new modalities and distributions~\cite{DBLP:journals/corr/abs-1902-00146}. We show that our model learns \textit{personalized} mood predictors from real-world private mobile data and better deals with \textit{heterogeneous} data never seen during training.

3) \textit{Fairness:} Real-world data often contains sensitive attributes and recent work has shown that it is possible to recover these attributes from data representations without access to the data itself~\cite{bolukbasi2016man}.
We show that local models can be modified to learn \textit{fair} representations that obfuscate protected attributes such as race, age, and gender, a feature crucial to preserving the privacy of on-device data.

\vspace{-2mm}
\section{Related Work}
\vspace{-1mm}

\textbf{Federated Learning} aims to train models in massively distributed networks~\cite{DBLP:journals/corr/McMahanMRA16} at a large scale~\cite{DBLP:journals/corr/abs-1902-01046}, over multiple sources of heterogeneous data~\cite{DBLP:journals/corr/abs-1812-06127}, and over multiple learning objectives~\cite{DBLP:journals/corr/SmithCST17}. Recent methods aim to improve the efficiency of federated learning~\cite{DBLP:journals/corr/abs-1812-07210}, perform learning in a one-shot setting~\cite{DBLP:journals/corr/abs-1902-11175}, propose realistic benchmarks~\cite{DBLP:journals/corr/abs-1812-01097}, and reduce the data mismatch between local and global data distributions~\cite{DBLP:journals/corr/abs-1902-00146}. While several specific algorithms have been proposed for heterogeneous data, \ours\ is a more \textit{general} framework that can handle heterogeneous data from new devices, reduce communication complexity, and ensure fair representation learning. We compare with these existing baselines and show that \ours\ outperforms them in heterogeneous settings.

\vspace{-0mm}
\textbf{Distributed Learning} is a related field with similarities and key differences: while both study the theory and practice involving the partition of data and aggregation of updates~\cite{DBLP:journals/corr/abs-1802-09941,pmlr-v70-suresh17a}, federated learning is additionally concerned with data that is private and distributed in a \textit{non-i.i.d.} fashion. Recent work has improved communication efficiency by sparsifying the data and model~\cite{pmlr-v70-wang17f}, developing efficient gradient methods~\cite{Wang2018CooperativeSA,NIPS2017_7218}, and compressing the updates~\cite{DBLP:journals/corr/abs-1802-06058}. These compression techniques are \textit{complementary} to our approach and can be applied to our local and global models.

\vspace{-0mm}
\textbf{Representation Learning} involves learning features from data for generative and discriminative tasks. A recent focus has been on learning \textit{fair} representations~\cite{Zemel:2013:LFR:3042817.3042973},
including using adversarial training~\cite{Goodfellow:2014:GAN:2969033.2969125} to learn representations that are not informative of private attributes~\cite{DBLP:journals/corr/abs-1904-13341,DBLP:journals/corr/abs-1901-10443} such as demographics~\cite{DBLP:journals/corr/abs-1808-06640} and gender~\cite{Wang2018AdversarialRO}. A related line of research is differential privacy which constraints statistical databases to limit the privacy impact on individuals whose information is in the database~\cite{Dwork:2006:DP:2097282.2097284,Dwork:2014:AFD:2693052.2693053}. 
Our approach extends recent advances in adapting federated learning for heterogeneous data and fairness. \ours\ is a \textit{general} framework that can handle heterogeneous data from new devices, reduce communication complexity, and ensure fair representation learning.

\vspace{-2mm}
\section{\ourl}
\vspace{-1mm}


At a high level, \ours\ combines local representation learning with global model learning in an \textit{end-to-end manner}. Each local device learns to extract higher-level representations from raw data before a global model operates on the representations (rather than raw data) from all devices. An overview of \ours\ is shown in Figure~\ref{local}. The local and global learning procedures are designed to be complementary: local representation learning aims to extract \textit{high level, compact features} important for prediction, thereby allowing the global model to save parameters by operating only on \textit{lower dimensional representations}. At the same time, the global model objective ensures that the global model must be able to classify data from \textit{all} devices, thereby ensuring that the local representations are general enough instead of overfitting to the subset of data on each device. We begin by describing how local (\S\ref{alg:local}) and global (\S\ref{alg:global}) learning is performed. We then detail one example of adversarial local learning to learn fair local representations (Appendix~\ref{alg:fair}).

\textbf{Notation:} We use uppercase letters $X$ to denote random variables and lowercase letters $x$ to denote their values. Upper case boldface letters $\mathbf{X}$ denote datasets consisting of multiple vector data points $\mathbf{x}$ which we represent by lowercase boldface letters. In the standard federated learning setting, we assume that we have data $\mathbf{X}_m \in \mathbb{R}^{N_m \times d}, m \in [M]$ and their corresponding labels $\mathbf{Y}_m \in \mathbb{R}^{N_m \times c}, m \in [M]$ across $M$ devices. $N_m$ denotes the number of data points on device $m$, $N = \sum_m N_m$ is the total number of data points, $d$ represents the input dimension and $c$ represents the number of classes for classification ($c=1$ for regression). Intuitively, each source of data captures a different view $p(X_m,Y_m)$ of the global data distribution $p(X,Y)$. In our experiments, we consider settings where the individual data points in $\mathbf{X}_m,\mathbf{Y}_m$ are sampled both i.i.d. and non i.i.d. with respect to $p(X,Y)$.
During training, we use parenthesized superscripts (e.g. $\theta^{(t)}$) to represent iteration $t$.

\begin{figure*}
    \begin{minipage}{\linewidth}
    \vspace{-4mm}
    \begin{algorithm}[H]
    \caption{\ours. The $M$ clients are indexed by $m$, $\eta$ is the learning rate.}
    \begin{algorithmic}[1]
        \SUB{\ Server executes:}
            \STATE initialize global model with weights $\theta^g$; initialize $M$ local models with weights  $\theta_m^\ell$.
            \STATE \textbf{for} each round $t = 1, 2, \dots$ \textbf{do}
            \STATE \quad $m \leftarrow \max(\clientfrac\cdot M, 1)$; $S_t \leftarrow$ (random set of $m$ clients)
            \STATE \quad \textbf{for} each client $m \in S_t$ \textbf{in parallel do}
            \STATE \quad \quad ${\theta^{g(t+1)}_m} \leftarrow \text{ClientUpdate}(m, {\theta^{g(t)}})$
            \STATE \quad ${\theta^{g(t+1)}} \leftarrow \sum_{m=1}^M \frac{N_m}{N} \theta^{g(t+1)}_m$ \qquad \ \ \ \ // aggregate updates
        \vspace{2mm}
        \SUB{\ ClientUpdate ($m, \theta_m^g$):}\ \ \qquad \qquad \ \ \ \ \ \ \ // run on client $m$
            \STATE $\mathcal{B} \leftarrow$ (split local data $(\mathbf{X}_m,\mathbf{Y}_m)$ into batches)
            \STATE \textbf{for} each local epoch \textbf{do}
            \STATE \quad \textbf{for} batch $(\mathbf{X},\mathbf{Y}) \in \mathcal{B}$ \textbf{do}
            \STATE \quad \quad $\mathbf{H} = \ell_m(\mathbf{X} ; \theta_m^\ell)$, $\hat{\mathbf{Y}} = g(\mathbf{H} ; \theta_m^g)$ \ \ \ \ \ // inference steps
            \STATE \quad \quad $\theta_m^\ell \leftarrow \theta_m^\ell - \eta \grad_{\theta_m^\ell} \mathcal{L}_m^g (\theta_m^\ell, \theta_m^g)$ \ \ \ \ \ \ // update local model wrt global loss
            \STATE \quad \quad $\theta_m^g \leftarrow \theta_m^g - \eta \grad_{\theta_m^g} \mathcal{L}_m^g (\theta_m^\ell, \theta_m^g)$ \ \ \ \ \ \ // update (local copy of) global model wrt global loss
            \STATE return global parameters $\theta_m^g$ to server
    \end{algorithmic}
    \label{algo}
    \end{algorithm}
    \end{minipage}
    \vspace{-6mm}
\end{figure*}

\vspace{-2mm}
\subsection{Local Representation Learning}
\label{alg:local}
\vspace{-1mm}

For each source of data $(\mathbf{X}_m,\mathbf{Y}_m)$, we learn a representation $\mathbf{H}_m$ which should: 1) be low-dimensional as compared to raw data $\mathbf{X}_m$, 2) capture important features in $\mathbf{X}_m$ that are useful towards the global model, and 3) not overfit to device data which may not align to the global data distribution. To be more concrete, we define features $\mathbf{z} \in \mathcal{Z}$ that should be captured using a good representation $\mathbf{h}$. In Figure~\ref{local}(a) through~\ref{local}(c) we summarize these local learning methods according to the choice of $\mathbf{z}$: (a) the labels $\mathbf{y}$ (supervised learning), (b) the data itself $\mathbf{x}$ (unsupervised autoencoder learning), or (c) some auxiliary labels $\mathbf{z}$ (self-supervised learning). For simplicity, we focus the description on supervised learning but describe extensions to local adversarial learning of fair representations (Figure~\ref{local}(d)) and unsupervised learning in Appendix~\ref{alg:fair}.

Each device consists of a local model $\ell_m: \mathbf{x} \rightarrow \mathbf{h}$ with parameters $\theta_m^\ell$ which allow us to infer features $\mathbf{H}_m = \ell_m(\mathbf{X}_m ; \theta_m^\ell)$ from local device data. These features should be useful in predicting the labels using a joint global model $g: \mathbf{h} \rightarrow \mathbf{y}$ with parameters $\theta^g$ over the features from all devices $\{\mathbf{H}_1, ..., \mathbf{H}_M\}$. The key difference is that the global model $g: \mathbf{h} \rightarrow \mathbf{y}$ now operates on lower-dimensional local representations $\mathbf{H}_m$. Therefore, $g$ can be a much smaller model which we will show in our experiments (\S\ref{efficient}).

\vspace{-2mm}
\subsection{Global Aggregation}
\label{alg:global}
\vspace{-1mm}

Learning this joint global model $g$ across all devices requires the aggregation of global parameter updates from each device. At each iteration $t$ of global model training, the server sends a copy of the global model parameters $\theta^{g(t)}$ to each device which we now label as $\theta_m^{g(t)}$ to represent the asynchronous updates made to each local copy. Each device runs their local model $\mathbf{H}_m = \ell_m(\mathbf{X}_m; \theta_m^\ell)$ to obtain local features and the global model $\hat{\mathbf{Y}}_m = g(\mathbf{H}_m ; \theta_m^{g(t)})$ to obtain predictions. We can compute the overall loss on device $m$:
{\fontsize{9.5}{12}\selectfont
\begin{align}
    \label{joint_obj}
    {\cal L}_m^g (\theta_m^\ell, \theta_m^g) = \mathbb{E}_{\substack{\mathbf{x} \sim X_m\\\mathbf{y} \sim Y_m|\mathbf{x}}} \left[ -\log \sum_{\mathbf{h}} \left( p_{\theta_m^g} (\mathbf{y}|\mathbf{h}) \ p_{\theta_m^\ell} (\mathbf{h}|\mathbf{x}) \right) \right].
\end{align}
}The loss is a function of both the local and global model parameters ($\theta_m^\ell$ and $\theta_m^g$ respectively) so both can be updated in an end-to-end manner. We argue that this synchronizes the training of the local models: while local models can flexibly fit the small amounts of data on their device, objective~\ref{joint_obj} acts as a regularizer to synchronize the local representations learned from all devices. Each local model cannot overfit to local data because otherwise, the joint global model would not be able to simultaneously predict from all local representations and the value of objective~\ref{joint_obj} would be high.

The local model parameters $\theta_m^\ell$ are updated by gradient-based methods in a straightforward manner. The global model parameters on device $m$, $\theta_m^{g(t)}$, are also asynchronously updated to $\theta_m^{g(t+1)}$ using gradient-based methods. After these local updates, each device now returns the updated global parameters $\theta_m^{g(t+1)}$ back to the server which aggregates these updates using a weighted average of the fraction of data points in each device, $\theta^{g(t+1)} = \sum_{m=1}^M \frac{N_m}{N} \theta^{g(t+1)}_m$~\citep{DBLP:journals/corr/McMahanMRA16}.
The overall training procedure is shown in Algorithm~\ref{algo}. Communication only happens when training the global model, which as we will show in our experiments, can be small given good local representations.


\vspace{-2mm}
\subsection{Inference at Test Time}
\label{alg:test}
\vspace{-1mm}

Given a new test sample $\mathbf{x}'$, how do we know which trained local model $\ell^*_m$ fits $\mathbf{x}'$ best? We consider two settings: (1) \textbf{Local Test} where we know which device the test data belongs to (e.g. training a personalized text completer). Using that local model works best for best match between train and test distributions. (2) \textbf{New Test} where it is possible to have a new device during testing with new data distributions. To address the new device, we view each local model as trained on a different view of the global data distribution. We pass $\mathbf{x}'$ through all trained local models $\ell^*_m$ and \textit{ensemble} the outputs.

\vspace{-2mm}
\section{Theoretical Analysis}
\vspace{-1mm}

In this section, we provide a theoretical analysis of using local and global models for federated learning. We show that 1) purely local models do not suffer from \textit{device variance} but suffer from \textit{data variance}, 2) the opposite holds true for purely global models, and 3) having both local and global models achieves a balance between both desiderata. All detailed proofs can be found in Appendix~\ref{sec:theory}. The link between the analysis and \ours\ for deep networks is discussed at the end of the section and verified via comprehensive experiments in \S~\ref{theory}.

We assume a student-teacher setting~\cite{krogh1992generalization,hastie2019surprises}, where the goal is to train a network $f_{\hat{\mathbf{u}}}$ with weights $\hat{\mathbf{u}} \in \mathbb{R}^d$ on a task whose target is produced by a teacher network $f_\mathbf{u}$ (i.e. the \textit{realizability assumption}). We assume that the targets are generated from a linear model. While simple in setup, its behavior is rich and has proved insightful in the understanding of nonlinear models as well~\cite{mei2019generalization,mei2019mean,hastie2019surprises}. 
To adapt this setting for federated learning, we assume that all device share some underlying structure (e.g. natural syntactic and semantic structures in text) while also displaying personalization across users (e.g. personalized vocabularies and writing styles). Mathematically, this involves a global feature vector $\mathbf{v}$ that represents shared features across devices as well as local features $\mathbf{r}_m$ that represent differences across devices. The labels on device $m$ are generated by a local teacher with weights $\mathbf{u}_{m} = \mathbf{v} + \mathbf{r}_m \in \mathbb{R}^d$. We assume that each local feature $\mathbf{r}_m \sim \mathcal{N}(0, \rho^2 I)$ is a different independent draw from a $d$-dimensional Gaussian with diagonal covariances of $\rho^2$.
$\rho^2$ represents \textit{device variance}: with higher $\rho^2$, the local features differ more representing more personalized targets across devices. We also assume that the training targets are corrupted by noise $\epsilon \sim \mathcal{N}(0, \sigma^2)$ to account for naturally-occurring \textit{data variance}~\citep{Montgomery}. Under this model, the training targets on model $m$ are given by $\tilde{f}_{\mathbf{u}_{m}} (\mathbf{x}) =
{f}_{\mathbf{u}_{m}} (\mathbf{x}) + \epsilon$. For simplicity, we assume each device contains $N$ data points.

Given $N$ datapoints $\{\mathbf{x}_i\}_{i=1,...,N}$, learning local and global parameters $\hat{\mathbf{u}}_m, \hat{\mathbf{v}}$ involve optimizing the following training objectives:
\begin{align}
    \label{learn_u}
    \hat{\mathbf{u}}_m &= \argmin_{\mathbf{w}} \frac{1}{N}\sum_{i=1}^N \left(f_{\mathbf{w}} (\mathbf{x}_i) - \tilde{f}_{\mathbf{u}_m}(\mathbf{x}_i) \right)^2, \hat{\mathbf{v}} = \argmin_{\mathbf{w}} \frac{1}{N} \sum_{m=1}^M \sum_{i=1}^N \left(f_{\mathbf{w}} (\mathbf{x}_i) - \tilde{f}_{\mathbf{u}_m}(\mathbf{x}_i) \right)^2.
\end{align}
We denote the overall model as $f(\mathbf{x}; \hat{\mathbf{v}}, \hat{\mathbf{u}}_{m})$ using both local and global models. The \textit{local empirical generalization error} on device $m$ is defined as
\begin{align}
    \hat{\mathcal{E}}_m &= \frac{1}{N}\sum_{i=1}^{N} \left( f(\mathbf{x}_i; \hat{\mathbf{v}}, \hat{\mathbf{u}}_{m}) - {f}_{\mathbf{u}_m}(\mathbf{x}_i) \right)^2.
\end{align}
The total \textit{empirical} generalization error is defined as the mean of all local errors $\hat{\mathcal{E}} = \frac{1}{M} \sum_{m=1}^M \hat{\mathcal{E}}_m$ and the true \textit{generalization error} is defined as the expectation taken over the randomness present in the data, devices, and noise:
{\fontsize{9}{12}\selectfont
\begin{align}
\label{eq: generalization}
    \mathcal{E} = \mathbb{E}_{\mathbf{x}, \mathbf{r}_m, \epsilon} [\hat{\mathcal{E}}_m] = \mathbb{E}_{\mathbf{x}, \mathbf{r}_m, \epsilon} \left[\left( f(\mathbf{x}; \hat{\mathbf{v}}, \hat{\mathbf{u}}_{m}) - \tilde{f}_{\mathbf{u}_m}(\mathbf{x}) \right)^2\right]
\end{align}
}which can further be manipulated to obtain a \textit{bias-variance decomposition} for federated learning.
\begin{theorem}
    The generalization loss for federated learning can be decomposed as
    \begin{equation}
    \label{eq: standard decomposition}
        \mathcal{E} = \mathbb{E}_{\mathbf{x}, \mathbf{r}_m, \epsilon}[\hat{\mathcal{E}}_m] = \text{Var}[\hat{f}] + b^2
    \end{equation}
    with variance $\text{Var}[\hat{f}] = \mathbb{E}_{\mathbf{x}, \mathbf{r}_m} \left[\text{Var}_{\epsilon}[\hat{f}| \mathbf{x}, \mathbf{r}_m] \right]$ and bias $b^2 = \mathbb{E}\left[\left( f_{\mathbf{u}_m} - \mathbb{E}_\epsilon f(\hat{\mathbf{v}}, \hat{\mathbf{u}}_{m})\right)^2\right]$.
\end{theorem}
Using only local models results in an unbiased estimator of $\mathbf{u}_{m}$. The bias term arises when learning global parameters since federated learning couples the estimation of both local and global parameters. The variance term comes from both the variance of both local and global parameter estimates.

As a simplified version, \ours\ can be seen as a ensemble of local and global models, i.e. $f(\mathbf{x}; \hat{\mathbf{v}}, \hat{\mathbf{u}}_{m}) = \alpha f_{\hat{\mathbf{u}}_m}(\mathbf{x}) + (1-\alpha) f_{\hat{\mathbf{v}}}(\mathbf{x})$. In this case, one can show that:
\begin{proposition}
\label{theo: bias-variance tradeoff}
Let $\mathbb{E}_{\mathbf{r}_m,\epsilon}[f_{\hat{\mathbf{v}}}] = f_\mathbf{v}$, $\mathbb{E}_\epsilon[f_{\hat{\mathbf{u}}_m}] = f_{\mathbf{u}_m}$, and let $f(\mathbf{x}; \hat{\mathbf{v}}, \hat{\mathbf{u}}_{m}) = \alpha f_{\hat{\mathbf{u}}_m}(\mathbf{x}) + (1-\alpha) f_{\hat{\mathbf{v}}}(\mathbf{x})$, then equation~\ref{eq: standard decomposition} can be written as
\begin{equation}
   \mathcal{E} = (1 - \alpha)^2 \delta^2 + \textup{Var}[\hat{f}],
\end{equation}
where $\delta^2 = \mathbb{E}_{\mathbf{x}, \mathbf{r}_m}[(f_\mathbf{u} - \mathbb{E}_{\epsilon}[f_\mathbf{v}])^2|\mathbf{r}_m]]$ measures the discrepancy between the local and global features as a result of local variations across devices.
\end{proposition}
\vspace{-1mm}
For the linear setting we are considering, the above result can be further expanded as:
\begin{equation}
    \mathcal{E} = (1 - \alpha)^2 \left(\frac{M-1}{M} \right)\rho^2 + \textup{Var}[\hat{f}].
\end{equation}

\vspace{-1mm}
\textbf{Analysis of local and global baselines:} We first analyze two baselines for federated learning. The first method learns local models on each device~\citep{DBLP:journals/corr/SmithCST17}: $f_{\ell}(\mathbf{x}; \hat{\mathbf{v}}, \hat{\mathbf{u}}_{m}) = f(\mathbf{x}; \hat{\mathbf{u}}_{m}) = \hat{\mathbf{u}}_{m}^\top \mathbf{x}$.
\begin{proposition}
The generalization error $\mathcal{E}(f_\ell)$ of local model $f_\ell(\mathbf{x}; \hat{\mathbf{v}}, \hat{\mathbf{u}}_{m})$ is $\frac{d}{N}\sigma^2$.
\end{proposition}
\vspace{-1mm}
This shows that local models only control data variance at a rate of $d/{N}$ since they are only updated using local device data which may be limited in number and vary highly (both in quality and quantity) across devices. However, local models \textit{do not suffer from device variance}.

The second method updates a joint global model (i.e. vanilla federated learning;~\cite{DBLP:journals/corr/McMahanMRA16}), which is equivalent to setting $\alpha=0$, i.e. $f_{g}(\mathbf{x}; \hat{\mathbf{v}}, \hat{\mathbf{u}}_{m}) = f(\mathbf{x}; \hat{\mathbf{v}}) = \hat{\mathbf{v}}^\top \mathbf{x}$. Its generalization error $\mathcal{E}(f_g)$ is:
\begin{proposition}
The generalization error $\mathcal{E}(f_g)$ of the global model $f_{g}(\mathbf{x}; \hat{\mathbf{v}}, \hat{\mathbf{u}}_{m})$ is $\frac{M-1}{M}\rho^2 + \frac{d}{MN}\sigma^2$.
\end{proposition}
\vspace{-1mm}
Global models can control for data variance ($\sigma^2$) at a rate of $d/(MN)$, decreasing with the total number of datapoints across \textit{all} devices (since global parameters are updated using data across all devices), which is better than the rate for local models. However, it suffers from an extra $O(\rho^2)$ term representing device variance so one global model is unable to account for very different devices.

\vspace{-1mm}
\textbf{Analysis of \ours:} Given that the above baselines achieve different generalization errors, one should be able to interpolate between the two methods to find the optimal tradeoff point. Therefore, our method defines an $\alpha-$interpolation between the global and local models, $f_{\alpha}(\mathbf{x}; \hat{\mathbf{v}}, \hat{\mathbf{u}}_{m}) = \alpha f_{\ell}(\mathbf{x}; \hat{\mathbf{u}}_{m}) + (1-\alpha)f_{g}(\mathbf{x}; \hat{\mathbf{v}})$, where $\alpha \in [0,1]$. The generalization error, $\mathcal{E}(f_{\alpha})$ is:
\begin{theorem}
    The generalization error $\mathcal{E}(f_{\alpha})$ is $\alpha^2\frac{d}{N}\sigma^2 + (1-\alpha)^2 \frac{M-1}{M}\rho^2 + (1-\alpha^2) \frac{d}{MN}\sigma^2$.
\end{theorem}
\begin{corollary}
The optimal $\alpha^*$ minimizing $\mathcal{E}(f_{\alpha})$ is $\alpha^* = \frac{\rho^2}{\rho^2 + \frac{d}{N}\sigma^2}$. When $\rho^2,\ \sigma^2 >0$, we have that $\mathcal{E}(f_{\alpha^*}) < \mathcal{E}(f_\ell)$ and $\mathcal{E}(f_{\alpha^*}) < \mathcal{E}(f_g)$, a generalization error better than local or global extremes.
\end{corollary}
This shows that using an ensemble of local and global models \textit{reduces both data variance and device variance}. When $\rho^2$ is large (high device variance), one should prioritize local models that better model the local data distributions (larger $\alpha^*$). Conversely, when $\sigma^2$ is large (high data variance), one should prioritize a global model (smaller $\alpha^*$).

While our theory holds true for linear models, we believe that it provides accurate insight into the practical generalization abilities of \ours, where we use deep networks and treat $\alpha$ as the \textit{split} of the layers of between the local and global models. Empirically, we compare Figure~\ref{test_plot}(a)-(b) (test error for linear models using $\alpha$-\textit{interpolation} on synthetic data) with Figure~\ref{test_plot}(d) (test accuracy for deep networks using $\alpha$-\textit{split} on real-world mobile data). The close similarity implies that our theoretical analysis captures the correct relationship between local and global models.

\vspace{-2mm}
\section{Experiments}
\vspace{-1mm}

We evaluate how our method 1) verifies our theory under different data and device variances, 2) \textit{efficiently} reduces parameters while retaining performance, 3) learns \textit{personalized} models and handles data from \textit{heterogeneous} sources, two settings with particularly high device variance, and 4) learns \text{fair} local representations that obfuscate private attributes. Code is included in the supplementary. Implementation details and sensitivity reports across hyperparameters are provided in Appendix~\ref{details_supp}.


\begin{figure}[t]
    \centering
    \vspace{-4mm}
    \includegraphics[width=1.0\linewidth]{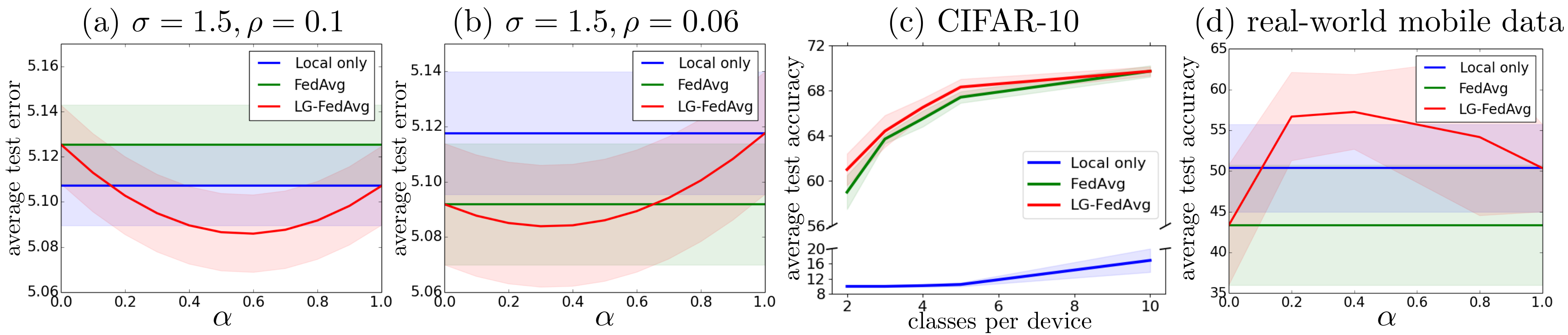}
    \caption{\small Test error (with shaded std dev) on synthetic data when local models perform better (plot (a): $\sigma=1.5,\rho=0.1$) and when global models perform better (plot (b): $\sigma=1.5,\rho=0.06$). For both settings, using an $\alpha$-interpolation of local and global models performs better than either extremes. (c): We also verify these theoretical findings on increasing device variance when splitting CIFAR-10 (fewer classes per device), where \ours\ consistently outperforms local only and \fed. (d): On predicting personalized moods from real-world private mobile data, an $\alpha$-split across local and global models outperforms either extremes.\vspace{-2mm}}
    \label{test_plot}
\end{figure}

\vspace{-2mm}
\subsection{Verifying Theoretical Analysis}
\label{theory}
\vspace{-1mm}






\textbf{Synthetic Data:} We first experiment on synthetic data to verify our theoretical analysis on ensembling local and global models. Data on device $m$ is generated by $\mathbf{x} \sim \mathcal{U}[-1.0, 1.0]$ and teacher weights $\mathbf{u}_m = \mathbf{v} + \mathbf{r}_m$ are sampled as $\mathbf{v} \sim \mathcal{U}[0.0, 1.0]$, $\mathbf{r}_m \sim \mathcal{N}(\mathbf{0}_d, \rho^2 \mathbf{I}_d)$, where $\rho^2$ represents device variance. Labels are observed with noise, $y = \mathbf{u}_m^\top \mathbf{x} + \epsilon$, $\epsilon \sim \mathcal{N}(0, \sigma^2)$, where $\sigma^2$ represents data variance. We plot the average test error when local models perform better due to higher device variance (Figure~\ref{test_plot}(a), $\sigma=1.5,\rho=0.1$) and when global models perform better due to lower device variance (Figure~\ref{test_plot}(b), $\sigma=1.5,\rho=0.06$). In Appendix~\ref{synthetic_supp}, we discuss the performance of \ours\ under several other variance settings. For all settings, using an $\alpha$-interpolation of both local and global models performs either close to the optimal extremes or better than either extreme.

\textbf{CIFAR-10:} Next, we verify our theory on deep networks over complex image classification problems using CIFAR-10. We focus on a highly \textit{non-i.i.d.} setting and follow the experimental design in~\cite{DBLP:journals/corr/McMahanMRA16} by assigning examples from at most $s \in \{2,3,4,5,10\}$ classes to each device.
The value of $s$ simulates device variance: $s=2$ represents highest device variance while $s=10$ represents an i.i.d. split of labels to devices. From Figure~\ref{test_plot}(c), we observe that \ours\ consistently outperforms local only and \fed. The performance gap is higher as device variance increases, which supports our theory that local models deal with high device variance.

\definecolor{gg}{RGB}{15,125,15}
\definecolor{rr}{RGB}{190,45,45}

\begin{table*}[!tbp]
\fontsize{8.5}{11}\selectfont
\centering
\caption{\small Comparison of federated learning methods on CIFAR-10 with non-iid split over devices. We report accuracy under both local test and new test settings as well as the total number of parameters communicated across training iterations. Best results in \textbf{bold}. \ours \ outperforms \fed\ and \textsc{MTL} under local test and achieves similar performance under new test while using around $50\%$ of the total parameters, and outperforms all using the same number of parameters. Mean and standard deviation are computed over $10$ runs.}
\setlength\tabcolsep{1.0pt}
\begin{tabular}{l l || c c c c c}
\Xhline{3\arrayrulewidth}
Data & \multirow{1}{*}{Method} &  \multirow{1}{*}{Local Test Acc. $(\uparrow)$} & \multicolumn{1}{c}{New Test Acc. $(\uparrow)$} & \multirow{1}{*}{FedAvg Rounds} & \multirow{1}{*}{LG Rounds} & \multirow{1}{*}{Params Comm. $(\downarrow)$} \\
\Xhline{0.5\arrayrulewidth}
\multirow{5}{*}{\rotatebox{90}{CIFAR-10}} & \fed~\citep{DBLP:journals/corr/McMahanMRA16} & $58.99 \pm 1.50$ & $58.99 \pm 1.50$ & $1800$ & $0$ & $12.7 \times 10^{9}$ \\
& Local only~\citep{DBLP:journals/corr/SmithCST17} & $87.93 \pm 2.14$ & $10.03 \pm 0.06$ & $0$ & $0$ & $0$ \\
& \textsc{MTL}~\citep{DBLP:journals/corr/SmithCST17} & $89.68 \pm 0.75$ & $10.06 \pm 0.11$ & $1800$ & $0$ & $12.0 \times 10^{9}$ \\
& \ours\ (ours) & $\mathbf{91.07 \pm 0.50}$ & $57.95 \pm 1.48$ & $1200$ & $100$ & $\mathbf{8.5 \times 10^{9}}$ \\
& \ours\ (ours) & $\mathbf{91.77 \pm 0.56}$ & $\mathbf{60.79 \pm 1.45}$ & $1800$ & $100$ & $12.7 \times 10^{9}$ \\
\Xhline{3\arrayrulewidth}
\end{tabular}
\label{cifar10}
\vspace{-6mm}
\end{table*}

\vspace{-2mm}
\subsection{Model Performance \& Communication Efficiency}
\label{efficient}
\vspace{-1mm}

\textbf{CIFAR-10:} We compare our approach with existing federated learning methods with respect to model performance and communication. We randomly assign each device to examples of two classes (highly unbalanced). We consider two settings during testing: 1) \textit{Local Test}, where we know which device the data belongs to (i.e. new predictions on an existing device) and choose that particular trained local model. For this setting, we split each device's data into train, validation, and test data, similar to~\cite{DBLP:journals/corr/SmithCST17}. 2) \textit{New Test}, in which we do not know which device the data belongs to (i.e. new predictions on new devices)~\cite{DBLP:journals/corr/McMahanMRA16}, so we use an ensemble approach by averaging all trained local model logits before choosing the most likely class~\cite{breiman1996bagging}. For ensembling, all local model weights are sent to the server \textit{only once} and averaged. We include this parameter exchange step and still show substantial communication improvement. We find that averaging model weights performs similarly ($0.5\%$ for CIFAR) to averaging model outputs since deep ReLU nets are mostly linear). We choose LeNet-5~\cite{Lecun98gradient-basedlearning} as our base model and we compare 4 methods: 1) \fed~\cite{DBLP:journals/corr/McMahanMRA16} which is the traditional federated learning approach, 2) Local only~\citep{DBLP:journals/corr/SmithCST17} as an extreme setting with only local models, 3) \textsc{MTL}~\citep{DBLP:journals/corr/SmithCST17} which trains local models with parameter sharing in a multi-task fashion, and 4) \ours\ which is our proposed method with local and global models.

The results in Table~\ref{cifar10} show that \textbf{\ours\ gives strong performance with low communication cost}. For CIFAR local test, \ours\ significantly outperforms \fed\ since local models allow us to better model the local device data distribution. For new test, \ours\ achieves similar performance to \fed\ while using around $50\%$ of the total parameters, and better performance with the same number of parameters. \ours\ also outperforms using local models only and local models trained with multitask learning (\textsc{MTL}). This shows that our end-to-end training strategy for local and global models is particularly suitable for large neural networks. We show MNIST results and sensitivity analysis wrt data and model splits in Appendix~\ref{efficient_supp} and find similar observations.

\definecolor{gg}{RGB}{15,125,15}
\definecolor{rr}{RGB}{190,45,45}

\begin{table*}[!tbp]
\fontsize{8.5}{11}\selectfont
\centering
\vspace{-4mm}
\caption{Comparison of \fed \ and \ours \ methods on VQA on non-i.i.d. device split setting.
\ours\ achieves strong performance while using fewer communicated parameters.\vspace{-2mm}}
\setlength\tabcolsep{3.0pt}
\begin{tabular}{l l || c c c c c}
\Xhline{3\arrayrulewidth}
Data & Method     &  Local Test Acc. $(\uparrow)$ & FedAvg Rounds & LG Rounds & Params Communicated $(\downarrow)$ \\
\Xhline{0.5\arrayrulewidth}
\multirow{2}{*}{VQA} & \fed~\cite{DBLP:journals/corr/McMahanMRA16} & $40.02$ & $47$ & $0$ & $13.97 \times 10^{10}$ \\
& \ours\ (ours) & $\mathbf{40.94}$ & $32$ & $17$ & $\mathbf{9.99 \times 10^{10}}$ \\
\Xhline{3\arrayrulewidth}
\end{tabular}
\label{vqa}
\vspace{-6mm}
\end{table*}

\textbf{Visual Question Answering (VQA):} VQA is a large-scale multimodal benchmark with $0.25$M images, $0.76$M questions, and $10$M answers~\cite{VQA}. We split the dataset in a non-i.i.d. manner and evaluate the accuracy under the local test setting. We use LSTM~\cite{hochreiter1997long} and ResNet-18~\cite{DBLP:journals/corr/HeZRS15} unimodal encoders as our local models and a global model which performs early fusion~\cite{DBLP:journals/corr/abs-1906-02125} of text and image features for answer prediction. In Table \ref{vqa}, we observe that \ours\ reaches a goal accuracy of $40\%$ while requiring lower communication costs (more VQA results and details in Appendix~\ref{vqa_supp}).

\vspace{-2mm}
\subsection{Learning Personalized Mood Predictors from Mobile Data}
\label{mood}
\vspace{-1mm}

\textbf{Data:} We designed and collected a new dataset, Mobile Assessment for the Prediction of Suicide (MAPS), to determine real-time indicators of suicide risk in adolescents aged $13-18$ years. This study monitors $100$ adolescents including individuals who have recently attempted suicide, individuals who experience suicidal ideation, and psychiatric controls. Across a duration of $6$ months, data was collected from each participant's smartphone using a keyboard logger which tracks all typed words. Participants were asked to rate their mood for the previous day on a scale ranging from $1-100$, with higher scores indicating a better mood. All users have given consent for their mobile device data to be collected and shared with us for research purposes. MAPS is a realistic federated learning benchmark since it contains real-world data with privacy concerns and high device variance due to highly personalized use of mobile phones. We used a preliminary preprocessed version containing $572$ samples across $14$ participants. We discretize the scores into $5$ bins for $5$-way classification. We use a random $80/10/10$ split for training/validation/testing, conduct all experiments $10$ times, and report the average accuracy and standard deviation (details in Appendix~\ref{mood_supp}).

\textbf{Models:} To assess how mobile text data can be used to make personalized mood predictions, we train a MLP classifier on top of a Bi-LSTM encoder on word embeddings.
In addition to local only and \fed, we test \ours\ across different splits of local and global model layers (i.e. $\alpha \in \{ 0.2, 0.4, 0.6, 0.8\}$) while keeping the total parameter count constant.

\textbf{Results:} From Figure~\ref{test_plot}(d), consistent with our theoretical findings, an $\alpha$-split across local and global models leverages both personalized representations per devices as well as statistical strength sharing through data across all devices, \textbf{outperforming either local or global extremes}.

\vspace{-2mm}
\subsection{Heterogeneous Data in an Online Setting}
\label{hetero}
\vspace{-1mm}

\definecolor{gg}{RGB}{15,125,15}
\definecolor{rr}{RGB}{190,45,45}

\begin{table*}[!tbp]
\fontsize{8.5}{11}\selectfont
\centering
\vspace{-6mm}
\caption{\small What happens when \fed \ trained on 100 devices of normal MNIST sees a device with rotated MNIST? Catastrophic forgetting, unless one fine-tunes again on training devices and incur high communication cost. \ours \ relieves catastrophic forgetting by using local models to perform well on both online rotated and regular MNIST, with $(C=0.1)$ and without $(C=0.0)$ fine-tuning.\vspace{-2mm}}
\setlength\tabcolsep{6.0pt}
\begin{tabular}{l l c || c c | c c }
\Xhline{3\arrayrulewidth}
\multirow{2}{*}{Data} & \multirow{2}{*}{Method} & \multirow{2}{*}{$C$} & \multicolumn{2}{c|}{{i.i.d. device data}} & \multicolumn{2}{c}{{non-i.i.d. device data}}\\
& & & Normal $(\uparrow)$ & Rotated $(\uparrow)$ & Normal $(\uparrow)$ & Rotated $(\uparrow)$ \\
\Xhline{0.5\arrayrulewidth}
\multirow{2}{*}{MNIST} & \multirow{1}{*}{\fed}~\cite{DBLP:journals/corr/McMahanMRA16} & $0.0$ & $32.0 \pm 6.2$ & $91.8 \pm 3.0$ & $35.7 \pm 4.3$ & $93.6 \pm 0.3$ \\
& \multirow{1}{*}{\ours} (ours) & $0.0$ & $\mathbf{96.6 \pm 0.9}$ & $\mathbf{92.9 \pm 2.7}$ & $\mathbf{96.3 \pm 0.3}$ & $\mathbf{94.1 \pm 0.7}$ \\
\Xhline{0.5\arrayrulewidth}
\multirow{3}{*}{MNIST} & \multirow{1}{*}{\fed}~\cite{DBLP:journals/corr/McMahanMRA16} & $0.1$ & $97.4 \pm 0.3$ & $89.3 \pm 0.8$ & $96.9 \pm 0.5$ & $89.6 \pm 0.6$ \\
& \multirow{1}{*}{\fedprox}~\cite{DBLP:journals/corr/abs-1812-06127} & $0.1$ & $94.8 \pm 1.1$ & $87.2 \pm 0.7$ & $97.9 \pm 0.1$ & $91.6 \pm 0.2$ \\
& \multirow{1}{*}{\ours} (ours) & $0.1$ & $\mathbf{97.7 \pm 0.8}$ & $\mathbf{93.2 \pm 1.3}$ & $\mathbf{98.2 \pm 0.7}$ & $\mathbf{93.9 \pm 1.4}$ \\
\Xhline{3\arrayrulewidth}
\end{tabular}
\label{mnist_rotated}
\vspace{-2mm}
\end{table*}

\textbf{Data:} We test whether \ours\ can handle heterogeneous data from a new source introduced during testing. We split MNIST across $100$ devices in both an {i.i.d.} and {non-i.i.d.} setting, then introduce a new device with $3,000$ training and $500$ test MNIST examples but \textit{rotated $90$ degrees}. This simulates a drastic change in the data distribution.

\textbf{Models:} We consider 3 methods: 1) \fed,
2) \textsc{FedProx}~\cite{DBLP:journals/corr/abs-1812-06127}: a method designed specifically for heterogeneous data by regularizing the local updates to reduce overfitting to local devices, and 3) \ours: train on the original $100$ devices, and when a new device comes, learn local representations before fine-tuning the global model. We hypothesize that good local models can ``unrotate'' images from the new device to better match the data distribution seen by the global model. When learning on the new device, we also retrain on a fraction $C$ of the original devices:
$C=0.0$ implies no fine-tuning and $C=0.1$ implies some fine-tuning ($C=1.0$ implies retraining on all data which is impractical).

\textbf{Results:} We report results in Table~\ref{mnist_rotated} and observe that: 1) \textbf{\fed\ suffers from catastrophic forgetting}~\cite{pmlr-v80-serra18a} without fine-tuning ($C=0.0$), in which the global model can perform well on the new device's rotated MNIST $(92\%)$ but completely forgets how to classify regular MNIST $(32\%)$. Only after fine-tuning ($C=0.1$) does the performance on both regular and rotated MNIST improve, but this requires more communication over the $100$ devices. 2) \textbf{\ours\ with local models relieves catastrophic forgetting}. Augmenting local models indeed helps to improve online performance on rotated MNIST $(93\%)$ while allowing the global model to retain performance on regular MNIST $(97\%)$, outperforming both \fed\ and \fedprox. We believe \ours\ achieves these results by learning a strong local representation that requires fewer updates from the trained global model.








\definecolor{gg}{RGB}{15,125,15}
\definecolor{rr}{RGB}{190,45,45}

\begin{table*}[!tbp]
\fontsize{8.5}{11}\selectfont
\centering
\vspace{-2mm}
\caption{Enforcing independence with respect to protected attributes \textit{race} and \textit{gender} on income prediction with the UCI dataset. \ours$+$Adv uses local models with adversarial (adv) training to remove information about protected attributes, at the expense of a small drop in classifier (class) accuracy of around $2-4\%$.\vspace{-2mm}}
\setlength\tabcolsep{0.8pt}
\begin{tabular}{l l || c c c | c c c}
\Xhline{3\arrayrulewidth}
\multirow{2}{*}{Data} & \multirow{2}{*}{Method} & \multicolumn{3}{c|}{{i.i.d. device data}} & \multicolumn{3}{c}{{non-i.i.d. device data}} \\
& & Class Acc $(\uparrow)$ & Class AUC $(\uparrow)$ & Adv AUC $(\downarrow)$ & Class Acc $(\uparrow)$ & Class AUC $(\uparrow)$ & Adv AUC $(\downarrow)$ \\
\Xhline{0.5\arrayrulewidth}
\multirow{3}{*}{UCI} & \multirow{1}{*}{\fed}~\citep{DBLP:journals/corr/McMahanMRA16} & $83.7 \pm 3.1$ & $89.4 \pm 1.9$ & $65.5 \pm 1.6$& $83.7 \pm 1.8$ & $88.7 \pm 1.2$ & $64.1 \pm 2.1$ \\
& \multirow{1}{*}{\ours} & $84.3 \pm 2.4$ & $89.0 \pm 2.2$ & $63.3 \pm 3.7$ & $81.1 \pm 1.6$ & $84.4 \pm 2.4$ & $62.7 \pm 2.5$ \\
& \multirow{1}{*}{\ours$+$Adv} & $82.1 \pm 1.0$ & $85.7 \pm 1.7$ & $\mathbf{50.1 \pm 1.3}$ & $80.1 \pm 2.0$ & $84.1 \pm 2.3$ & ${\fontseries{b}\selectfont 49.8 \pm 2.2}$ \\
\Xhline{3\arrayrulewidth}
\end{tabular}
\label{res:fair}
\vspace{-6mm}
\end{table*}

\vspace{-2mm}
\subsection{Learning Fair Representations}
\label{fair}
\vspace{-1mm}

\textbf{Data:} We examine whether local models can be trained to protect private attributes from the global model. We use the UCI adult dataset~\cite{Kohavi96scalingup} to predict whether an individual makes more than $50$K per year based on their personal attributes. However, we want our models to be invariant to the sensitive attributes of \textit{race} and \textit{gender} instead of picking up on correlations that could exacerbate biases.

\textbf{Models:} We adapt adversarial learning to remove protected attributes from local models (see Appendix~\ref{alg:fair}. Specifically, we aim to learn fair local representations from which a fully trained adversarial network should \textit{not} be able to predict the protected attributes.
We report three methods: 1) \fed\ with only a global model and global adversary both updated using \fed. The global model is not trained with the adversarial loss since it is simply not possible: once local device data passes through the global model, privacy is potentially violated. 2) \ours\ without penalizing the adversarial network, and 3) \ours$+$Adv which jointly trains local, global, and adversary models to learn fair local representations before global prediction.

\textbf{Results:} We report results according to: 1) classifier binary accuracy, 2) classifier ROC AUC score, and 3) adversary ROC AUC score. The classifier metrics should be as close to $100\%$ as possible while the adversary should be as close to $50\%$ as possible. From Table~\ref{res:fair}, \ours$+$Adv \textbf{learns fair local representations that are unable to predict protected attributes} ($\sim50\%$ adversary AUC) with \textbf{only a small drop in global accuracy} $(\sim4\%)$. In order to ensure that poor adversary AUC was indeed due to fair representations instead of a poorly trained adversary, we train a post-fit classifier from local representations to protected attributes and achieve similar random results.

\vspace{-2mm}
\section{Conclusion}
\vspace{-1mm}

We proposed \ours\ combining \textit{local representation learning} with federated training of global models. Our theoretical analysis shows that an ensemble of local and global models reduces both data variance and device variance. On a suite of real-world datasets, \ours\ achieves strong performance while reducing communication costs, learns personalized models, better deals with heterogeneous data, and effectively learns fair representations that obfuscate protected attributes.


\vspace{-2mm}
\section*{Acknowledgements}
\vspace{-1mm}

PPL and LM were partially supported by the National Science Foundation (Awards \#1750439, \#1722822) and National Institutes of Health. RS was supported in part by NSF IIS1763562, Office of Naval Research N000141812861, and Google focused award. Any opinions, findings, and conclusions or recommendations expressed in this material are those of the author(s) and do not necessarily reflect the views of National Science Foundation or National Institutes of Health, and no official endorsement should be inferred. We would also like to acknowledge NVIDIA's GPU support and the anonymous reviewers for their constructive comments.

\vspace{-2mm}
\section*{Broader Impacts}
\vspace{-1mm}

Federated learning provides tools for large-scale distributed training at unprecedented scales but at the same time requires more research on its implications to society and policy.

\textbf{Broader applications:} By 2025, it is estimated that there will be more than to 75 billion IoT (Internet of Things) devices all connected to the internet and sharing data with each other~\citep{DBLP:journals/corr/abs-1903-05266}. Organizing, processing, and learning from device data will use federated learning techniques. It has already been shown to be a promising approach for applications such as learning the social activities of mobile phone users, early forecasting of health events like heart attack risks from wearable devices, and localization of pedestrians for autonomous vehicles~\citep{li2019federated}. The societal impacts revolving around more invasive federated learning technologies have to be taken into account as we design future systems that increasingly leverage distributed mobile data.

\textbf{Applications in mental health:} Suicide is the second leading cause of death among adolescents. In addition to deaths, 16\% of high school students report seriously considering suicide each year, and 8\% make one or more suicide attempts (CDC, 2015). Despite these alarming statistics, there is little consensus concerning imminent risk for suicide~\citep{franklin2017risk,large2017patient}. Given the impact of suicide on society, there is an urgent need to better understand the behavior markers related to suicidal ideation.

``Just-in-time'' adaptive interventions delivered via mobile health applications provide a platform of exciting developments in low-intensity, high-impact interventions~\citep{nahum2018just}. The ability to intervene precisely during an acute risk for suicide could dramatically reduce the loss of life. To realize this goal, we need accurate and timely methods that predict when interventions are most needed. Federated learning is particularly useful in monitoring (with participants' permission) mobile data to assess mental health and provide early interventions. Our data collection, experimental study, and computational approaches provide a step towards data intensive longitudinal monitoring of human behavior. However, one must take care to summarize behaviors from mobile data without identifying the user through personal (e.g., personally identifiable information) or protected attributes (e.g., race, gender). This form of anonymity is critical when implementing these suicide detection technologies in real-world scenarios. Our goal is to be highly predictive of STBs while remaining as privacy-preserving as possible. We outline some of the potential privacy and security concerns below and show some possibilities brought about from the flexibility of our local models.

\textbf{Privacy:} There are privacy risks associated with making predictions from mobile data. Although federated learning only keeps data private on each device without sending it to other locations, the presence of one's data during distributed model training will likely affect model predictions. Therefore it is crucial to obtain user consent before collecting device data. In our experiments with real-world mobile data, all participants have given consent for their mobile device data to be collected and shared with us for research purposes. All data was anonymized and stripped of all personal (e.g., personally identifiable information) and protected attributes (e.g., race, gender).

\textbf{Security:} Communicating model updates throughout the training process could possibly reveal sensitive information, either to a third-party, or to the central server. Federated learning is also particularly sensitive to external security attacks from adversaries~\citep{lyu2020threats}. Recent methods to increase the security of federated learning systems come at the cost of reduced performance or efficiency. We believe that our proposed local-global models make federated learning more interpretable and flexible since local models can be appropriately adjusted to be more secure. However, there is a lot more work to be done in these directions, starting by accurately quantifying the trade-offs between security, privacy and performance in federated learning~\citep{li2019federated}.

\textbf{Social biases:} We acknowledge that there is a risk of exposure bias due to imbalanced datasets, especially when personal mobile data is involved. Models trained on biased data have been shown to amplify the underlying social biases especially when they correlate with the prediction targets~\citep{DBLP:journals/corr/abs-1809-07842}. Our experiment showcased one example of maintaining fairness via adversarial training, but leaves room for future work in exploring other methods tailored for specific scenarios (e.g. debiasing words~\citep{bolukbasi2016man}, sentences~\citep{debiasing}, and images~\citep{10.1145/3209978.3210094}). These methods can be easily applied to local representations before input into the global model. Future research should also focus on quantifying the trade-offs between bias and performance~\citep{DBLP:journals/corr/abs-1906-08386}.

Overall, we believe that our proposed approach can help quantify the tradeoffs between local and global models regarding performance, communication, privacy, security, and fairness. Its flexibility also offers several exciting directions of future work in ensuring privacy and fairness of local representations. We hope that this brings about future opportunities for large-scale real-time analytics in healthcare and transportation using federated learning.



{\small
\bibliography{main}
\bibliographystyle{plain}
}

\clearpage
\onecolumn 
\appendix

\section*{Appendix}

\vspace{-1mm}
\section{Theoretical Analysis}
\label{sec:theory}
\vspace{-1mm}

\addtocounter{theorem}{-2}
\addtocounter{proposition}{-4}
\addtocounter{corollary}{-2}

In this section, we restate our main results and provide details proofs for them. We start with the short comment that, for the federated linear regression problem we are solving, gradient descent converges to the same solution as the commonly used ridgeless linear regression solution, as is shown in \cite{hastie2019surprises}. This means that $\hat{\mathbf{u}}$ converges to $\mathbf{u}$ and $\hat{\mathbf{v}}$ converges to $\mathbf{v}$ in expectation. We also assume in our analysis that $\frac{d}{N}$ is small and can be summarized in the big-$O$ notation; however, this does not mean that the term $\frac{d}{N}\sigma^2$ is small because the noise present in the dataset might be a function of $N$. For example, in the well-studied label noise literature, the noise rate $\sigma^2$ is often proportional to $N$, cancelling out the $N$ term in the denominator~\cite{ziyin2020learning}. In practice, this happens when the dataset has a fixed probability of wrong labels.
 
\begin{theorem}
    The generalization loss for federated learning can be decomposed as
    \begin{equation}
    \label{eq: standard decomposition_supp}
        \mathcal{E} = \mathbb{E}_{\mathbf{x}, \mathbf{r}_m, \epsilon}[\hat{\mathcal{E}}_m] = \text{Var}[\hat{f}] + b^2
    \end{equation}
    where $\text{Var}[\hat{f}] = \mathbb{E}_{\mathbf{x}, \mathbf{r}_m} \left[\text{Var}_{\epsilon}[\hat{f}| \mathbf{x}, \mathbf{r}_m] \right]$ is the variance, and $b^2 = \mathbb{E}\left[\left( f_{\mathbf{u}_m} - \mathbb{E}_\epsilon f(\hat{\mathbf{v}}, \hat{\mathbf{u}}_{m})\right)^2\right]$ is the bias.
\end{theorem}

\begin{proof}
The \textit{generalization error} can be derived by taking expectation over the random variables:
\begin{align}
\label{eq: generalization_supp}
    \mathcal{E} = \mathbb{E}_{\mathbf{x}, \mathbf{r}_m, \epsilon} [\hat{\mathcal{E}}_m] = \mathbb{E}_{\mathbf{x}, \mathbf{r}_m, \epsilon} \left[\left( f(\mathbf{x}; \hat{\mathbf{v}}, \hat{\mathbf{u}}_{m}) - \tilde{f}_{\mathbf{u}_m}(\mathbf{x}) \right)^2\right].
\end{align}
which can further be manipulated to obtain a \textit{bias-variance decomposition} for federated learning:
\begin{align}
\label{eq: generalization_supp}
    \mathcal{E} &= \mathbb{E}_{\mathbf{x}, \mathbf{r}_m, \epsilon} [\hat{\mathcal{E}}_m] \\
    &= \mathbb{E}_{\mathbf{x}, \mathbf{r}_m, \epsilon} \left[\left( f(\mathbf{x}; \hat{\mathbf{v}}, \hat{\mathbf{u}}_{m}) - f_{\mathbf{u}_m}(\mathbf{x}) \right)^2\right]\\
    &= \mathbb{E} \left[\left( f(\hat{\mathbf{v}}, \hat{\mathbf{u}}_{m}) - \mathbb{E}f(\hat{\mathbf{v}}, \hat{\mathbf{u}}_{m}) - [f_{\mathbf{u}_m} - \mathbb{E}f(\hat{\mathbf{v}}, \hat{\mathbf{u}}_{m})] \right)^2\right]\\
    &= \mathbb{E}\left[\left( f(\hat{\mathbf{v}}, \hat{\mathbf{u}}_{m}) - \mathbb{E}f(\hat{\mathbf{v}}, \hat{\mathbf{u}}_{m})\right)^2\right] + 
    \mathbb{E}\left[\left( f_{\mathbf{u}_m} - \mathbb{E}f(\hat{\mathbf{v}} , \hat{\mathbf{u}}_{m})\right)^2\right] \\
    &= \underbrace{\mathbb{E}_{\mathbf{x}, \mathbf{r}_m}\left[\text{Var}_{\epsilon}[\hat{f}|\mathbf{x}, \mathbf{r}_m]\right]}_{\text{Var}[\hat{f}]\text{: variance of model}} + \underbrace{\mathbb{E}\left[\left( f_{\mathbf{u}_m} - \mathbb{E}_\epsilon f(\hat{\mathbf{v}} , \hat{\mathbf{u}}_{m})\right)^2\right]}_{b^2\text{: bias of model}} + \underbrace{0}_{\text{cross term}}
\end{align}
\end{proof}
where we have omitted $\mathbf{x}_i$ in the input to $f(\cdot)$ for notational conciseness. Using only local models results in an unbiased estimator of $\mathbf{u}_{m}$. The bias term arises when learning global model parameters since federated learning couples the estimation of both local and global parameters. The variance term comes from both the variance of both local and global parameter estimates.

As a simplified version, \ours\ can be seen as a ensemble of local and global models, i.e. $f(\mathbf{x}; \hat{\mathbf{v}}, \hat{\mathbf{u}}_{m}) = \alpha f_{\hat{\mathbf{u}}_m}(\mathbf{x}) + (1-\alpha) f_{\hat{\mathbf{v}}}(\mathbf{x})$. In this case, one can show that:
\begin{proposition}
\label{theo: bias-variance tradeoff_supp}
Let $\mathbb{E}_{\mathbf{r}_m,\epsilon}[f_{\hat{\mathbf{v}}}] = f_\mathbf{v}$, $\mathbb{E}_\epsilon[f_{\hat{\mathbf{u}}_m}] = f_{\mathbf{u}_m}$, and let $f(\mathbf{x}; \hat{\mathbf{v}}, \hat{\mathbf{u}}_{m}) = \alpha f_{\hat{\mathbf{u}}_m}(\mathbf{x}) + (1-\alpha) f_{\hat{\mathbf{v}}}(\mathbf{x})$, then equation~\ref{eq: standard decomposition_supp} can be written as
\begin{equation}
    \mathcal{E} = (1 - \alpha)^2 \delta^2 + \textup{Var}[\hat{f}]
\end{equation}
where $\delta^2 = \mathbb{E}_{\mathbf{x}, \mathbf{r}_m}[(f_\mathbf{u} - \mathbb{E}_{\epsilon}[f_\mathbf{v}])^2|\mathbf{r}_m]]$ measures the discrepancy between the local features and the global features, and thus measures the local variations across devices.
\end{proposition}
\begin{proof}
We plug in to get 
\begin{align}
    \mathbb{E}_{\mathbf{x}, \mathbf{r}_m}\left[\left( f_{\mathbf{u}_m} - \mathbb{E}_\epsilon f(\hat{\mathbf{v}} , \hat{\mathbf{u}}_{m})\right)^2\right] &= \mathbb{E}_{\mathbf{x}, \mathbf{r}_m} \left[\left( f_{\mathbf{u}_m} - \alpha f_{\mathbf{u}_{m}} - \mathbb{E}_\epsilon  [(1-\alpha)f_{\hat{\mathbf{v}}}]\right)^2\right]\\
    & = (1-\alpha)^2 \mathbb{E}_{\mathbf{x}, \mathbf{r}_m}\left[(f_{\mathbf{u}_m} -  \mathbb{E}_\epsilon  [f_{\hat{\mathbf{v}}}|\mathbf{r}_m])^2  \right]
\end{align}
\end{proof}
When using linear models, we can further expand this result as follows:
\begin{corollary}
\label{corr_final_decomp}
Let $f_{\hat{\mathbf{v}}}(\mathbf{x}) = \hat{\mathbf{v}}^\top\mathbf{x}$, $f_{\hat{\mathbf{u}}_m}(\mathbf{x}) = \hat{\mathbf{u}_m}^\top\mathbf{x}$ be learned through gradient descent algorithm, then $\delta^2 = (\frac{M-1}{M})\rho^2$, and $\textup{Var}[\hat{f}]$ can be expanded as follows: 
\begin{equation}
\label{final_decomp}
    \mathcal{E} = (1 - \alpha)^2 \left(\frac{M-1}{M}\right)\rho^2 + (1-\alpha)^2 \textup{ Var}[{f}_{\hat{\mathbf{v}}}] +  \alpha^2 \textup{ Var}[{f}_{\hat{\mathbf{u}}}] + 2\alpha(1-\alpha) \textup{ Cov}[f_{\hat{\mathbf{v}}}, f_{\hat{\mathbf{u}}}]
\end{equation}
\end{corollary}
\begin{proof}
We first show that
$$\mathbb{E}_{\mathbf{x}, \mathbf{r}_m}\left[(f_{\mathbf{u}_m} -  \mathbb{E}_\epsilon  [f_{\hat{\mathbf{v}}}|\mathbf{r}_m])^2  \right] = \left(\frac{M-1}{M} \right)\rho^2. $$
We expand to get:
\begin{align}
    \mathbb{E}_{\mathbf{x}, \mathbf{r}_m} \left[(f_{\mathbf{u}_m} - \mathbb{E}_\epsilon  [f_{\hat{\mathbf{v}}}|\mathbf{r}_m])^2 \right] &= \mathbb{E}_{\mathbf{x}, \mathbf{r}_m} \left[(\mathbf{u}_m^\top\mathbf{x} - \mathbb{E}_\epsilon [ \hat{\mathbf{v}}^\top\mathbf{x}|\mathbf{r}_m])^2 \right]\\
    &= \mathbb{E}_{\mathbf{x}, \mathbf{r}_m} \left[ \left(\mathbf{u}_m^\top\mathbf{x} - \mathbb{E}_\epsilon  [\frac{1}{M}\sum_{j=1}^M \hat{\mathbf{u}}_j^\top\mathbf{x}|\mathbf{r}_m] \right)^2 \right]\\
    &= \mathbb{E}_{\mathbf{x}, \mathbf{r}_m} \left[ \left(\mathbf{u}_m^\top\mathbf{x} - \frac{1}{M}\sum_{j=1}^M {\mathbf{u}}_j^\top\mathbf{x} \right)^2 \right]\\
    &= \mathbb{E}_{\mathbf{x}, \mathbf{r}_m} \left[ \frac{1}{{M}^2} \left( \sum_{j\neq m}^M \mathbf{u}_m^\top\mathbf{x} - {\mathbf{u}}_j^\top\mathbf{x} \right)^2 \right]\\
    &= \mathbb{E}_{\mathbf{x}, \mathbf{r}_m} \left[ \frac{1}{{M}^2} \left( \sum_{j\neq m}^M \mathbf{r}_m^\top\mathbf{x} - {\mathbf{r}}_j^\top\mathbf{x} \right)^2 \right]\\
    &= \mathbb{E}_{\mathbf{x}, \mathbf{r}_m} \left[ \frac{1}{{M}^2} \left( \sum_{j\neq m}^M \mathbf{r}_m^\top\mathbf{x} - {\mathbf{r}}_j^\top\mathbf{x} \right)^2 \right]\\
    &= \text{Tr} \left\{\left[ \frac{1}{{M}^2} \mathbb{E}_{\mathbf{r}_m}\left( \sum_{j\neq m }^M \mathbf{r}_m^\top- {\mathbf{r}}_j^\top \right)^2 \right] \mathbb{E}_{\mathbf{x}} [\mathbf{x}\mathbf{x}^\top] \right\}\\
    &= \frac{M(M-1)}{M^2}\rho^2\\
    &= \left(\frac{M-1}{M}\right)\rho^2
\end{align}
where we have used the fact that the expectation over squared $\mathbf{x}$, $\mathbb{E}_{\mathbf{x}} [\mathbf{x}\mathbf{x}^\top]$, is the identity matrix.

Furthermore, by using the fact that $f(\mathbf{x}; \hat{\mathbf{v}}, \hat{\mathbf{u}}_{m}) = (1-\alpha) f_{\hat{\mathbf{u}}_m}(\mathbf{x}) + \alpha f_{\hat{\mathbf{v}}}(\mathbf{x})$, we have that
\begin{align}
    \textup{Var}[\hat{f}] = (1-\alpha)^2 \textup{ Var}[{f}_{\hat{\mathbf{v}}}] + \alpha^2 \textup{ Var}[{f}_{\hat{\mathbf{u}}}] + 2\alpha(1-\alpha) \textup{ Cov}[f_{\hat{\mathbf{v}}}, f_{\hat{\mathbf{u}}}]
\end{align}
\end{proof}
Using these results, we can compute the generalization error of several federated learning baselines as well as \ours.

\vspace{-1mm}
\subsection{Analysis of Federated Learning Baselines}
\vspace{-1mm}

We begin by analyzing two baselines for federated learning. 

The first method learns local models on each device~\citep{DBLP:journals/corr/SmithCST17}: $f_{\ell}(\mathbf{x}; \hat{\mathbf{v}}, \hat{\mathbf{u}}_{m}) = f(\mathbf{x}; \hat{\mathbf{u}}_{m}) = \hat{\mathbf{u}}_{m}^\top \mathbf{x}$.
\begin{proposition}
\label{error_local}
The generalization error $\mathcal{E}(f_\ell)$ of local model $f_\ell(\mathbf{x}; \hat{\mathbf{v}}, \hat{\mathbf{u}}_{m})$ is $\frac{d}{N}\sigma^2$.
\end{proposition}
\begin{proof}
Similarly, set $\alpha=1$ in equation~\ref{final_decomp} of corollary~\ref{corr_final_decomp} and we obtain that $\mathcal{E}(f_\ell) = \text{Var}[f_{\hat{\mathbf{u}}}] = \frac{d}{N}\sigma^2$.
\end{proof}
This shows that local models only control data variance at a rate of $d/{N}$ since they are only updated using local device data which may be limited in number and vary highly (both in quality and quantity) across devices. However, local models do not suffer from device variance.

The second method updates a joint global model~\citep{DBLP:journals/corr/McMahanMRA16}, which is equivalent to setting $\alpha=0$ in our previous analysis, i.e. $f_{g}(\mathbf{x}; \hat{\mathbf{v}}, \hat{\mathbf{u}}_{m}) = f(\mathbf{x}; \hat{\mathbf{v}}) = \hat{\mathbf{v}}^\top \mathbf{x}$. We can compute its generalization error:
\begin{proposition}
\label{error_global}
The generalization error $\mathcal{E}(f_g)$ of the global model $f_{g}(\mathbf{x}; \hat{\mathbf{v}}, \hat{\mathbf{u}}_{m})$ is $\frac{M-1}{M}\rho^2 + \frac{d}{MN}\sigma^2$.
\end{proposition}
\begin{proof}
Set $\alpha=0$ in equation~\ref{final_decomp} of corollary~\ref{corr_final_decomp}, we obtain $\mathcal{E}(f_g) = \frac{M-1}{M} \rho^2 + \textup{Var}[f_{\hat{\mathbf{v}}}] = \frac{M-1}{M}\rho^2 + \frac{d}{MN}\sigma^2$.
\end{proof}
Therefore, global models are able to control for data variance ($\sigma^2$) at a rate of $d/(MN)$, decreasing with the total number of datapoints across \textit{all} devices (since global parameters are updated using data across all devices), which is better than the rate for local models. However, it suffers from an extra $O(\rho^2)$ term representing device variance so one global model is unable to account for very different devices.

\vspace{-1mm}
\subsection{Analysis of \ours}
\vspace{-1mm}

Given that the above baselines achieve different generalization errors, one should be able to interpolate between the two methods to find the optimal tradeoff point. Our method therefore defines an $\alpha-$interpolation between the local and global models:
\begin{align}
    f_{\alpha}(\mathbf{x}; \hat{\mathbf{v}}, \hat{\mathbf{u}}_{m}) &= \alpha f_\ell(\mathbf{x}; \hat{\mathbf{u}}_{m}) + (1-\alpha)f_g(\mathbf{x}; \hat{\mathbf{v}}).
\end{align}
where $\alpha \in [0,1]$. The following theorem gives the generalization of this model.
\begin{theorem}
    The generalization error of $f_{\alpha}(\mathbf{x}; \hat{\mathbf{v}}, \hat{\mathbf{u}}_{m})$ is $\alpha^2\frac{d}{N}\sigma^2 + (1-\alpha)^2 \frac{M-1}{M}\rho^2 + (1-\alpha^2) \frac{d}{MN}\sigma^2$.
\end{theorem}
\begin{proof}
The proof follows by computing each term in equation~\ref{final_decomp} of corollary~\ref{corr_final_decomp}
\begin{equation}
    \mathcal{E}(f_\alpha) = (1 - \alpha)^2 \left(\frac{M-1}{M}\right)\rho^2 + (1-\alpha)^2 \textup{ Var}[{f}_{\hat{\mathbf{v}}}] +  \alpha^2 \textup{ Var}[{f}_{\hat{\mathbf{u}}}] + 2\alpha(1-\alpha) \textup{ Cov}[f_{\hat{\mathbf{v}}}, f_{\hat{\mathbf{u}}}]
\end{equation}
We need to find $\textup{Cov}[f_{\hat{\mathbf{v}}}, f_{\hat{\mathbf{u}}}]$:
\begin{align}
    \textup{Cov}[f_{\hat{\mathbf{v}}}, f_{\hat{\mathbf{u}}}] &=
    \mathbb{E} \left[(f_{\hat{\mathbf{v}}} - \mathbb{E} f_{\hat{\mathbf{v}}}) (f_{\hat{\mathbf{u}}} - \mathbb{E} f_{\hat{\mathbf{u}}}) \right] \\
    &= \mathbb{E} \left[ \left(\frac{1}{M}\sum_{j=1}^M \hat{\mathbf{u}}_j - \mathbf{v} \right)^\top (\hat{\mathbf{u}}_m - \mathbf{u}_m) \right] \\
    &= \mathbb{E} \left[\frac{1}{M}\sum_{j=1}^M \hat{\mathbf{r}}_j^\top \hat{\mathbf{r}}_m \right] \\
    &= \frac{d}{MN}\sigma^2 
\end{align}
where, as in the previous theorem, the expectation over squared $\mathbf{x}$, $\mathbb{E}_{\mathbf{x}} [\mathbf{x}\mathbf{x}^\top]$, is the identity matrix. Likewise, we obtain $\textup{Var}[{f}_{\hat{\mathbf{u}}}] = \frac{d}{N}\sigma^2$, and $\textup{Var}[{f}_{\hat{\mathbf{v}}}] = \frac{d}{MN}\sigma^2$. Combining everything above, we get that the generalization error is
\begin{align}
    \mathcal{E}(f_{\alpha}) &= \alpha^2\frac{d}{N}\sigma^2 + (1-\alpha)^2 \frac{M-1}{M}\rho^2 + (1-\alpha)^2 \frac{d}{MN}\sigma^2 + 2 \alpha (1-\alpha)\frac{d}{MN}\sigma^2 \\
    &= \alpha^2\frac{d}{N}\sigma^2 + (1-\alpha)^2 \frac{M-1}{M}\rho^2 + (1-\alpha^2) \frac{d}{MN}\sigma^2
\end{align}
\end{proof}
This can be solved to find the optimal $\alpha^*$ that minimizes $f_{\alpha}(\mathbf{x}; \hat{\mathbf{v}}, \hat{\mathbf{u}}_{m})$.
\begin{corollary}
The optimal $\alpha^*$ that minimizes $\mathcal{E}(f_{\alpha})$ is
\begin{equation}
    \alpha^* = \frac{\rho^2}{\rho^2 + \frac{d}{N}\sigma^2}.
\end{equation}
Moreover, when $\sigma^2, \rho^2 \neq 0$, we have that $\mathcal{E}(f_{\alpha^*}) < \mathcal{E}(f_\ell)$ and $\mathcal{E}(f_{\alpha^*}) < \mathcal{E}(f_g)$.
\end{corollary}
\begin{proof}
Taking derivative w.r.t to $\alpha$, and setting to $0$ gives an exact expression for $\alpha^*$:
\begin{equation}
    \alpha^* = \frac{\frac{M-1}{M}\rho^2}{\frac{M-1}{M}\rho^2 + \frac{M-1}{M} \frac{d}{N}\sigma^2} = \frac{\rho^2}{\rho^2 + \frac{d}{N}\sigma^2}.
\end{equation}
\end{proof}
This shows that using an ensemble of local and global models reduces both data variance and device variance. When $\rho^2$ is large (high device variance), one should lean towards using local models that better model the local data distributions (larger $\alpha^*$). Conversely, when $\sigma^2$ is large (high data variance), one should lean towards using a global model whose parameters are updated using more data across all devices (smaller $\alpha^*$).

\vspace{-1mm}
\section{Fair Representation Learning}

\vspace{-2mm}
\subsection{Fair Training of Local Models}
\label{alg:fair}
\vspace{-1mm}

In this section we detail one extension of local representation learning to remove information that might be indicative of protected attributes. The data on each device is now a triple $(\mathbf{X}_m, \mathbf{Y}_m, \mathbf{P}_m)$ drawn non-i.i.d. from a joint distribution $p(X,Y,P)$ where $\mathbf{p} \in \mathcal{P}$ are some protected attributes which the model should not predict. For example, although there exist correlations between race and income~\cite{10.2307/1054978} which could help in income prediction~\cite{NIPS2018_7613}, it would be undesirable for our models to rely on these correlations since these would exacerbate racial biases.

To learn fair local representations, we use adversarial training~\cite{NIPS2017_6699} to remove protected attributes (Figure~\ref{local} (d)). More formally, we aim to learn a local model $\ell_m$ such that the distribution of $\ell_m(\mathbf{x} ; \theta_m^\ell)$ conditional on $\mathbf{h}$ is invariant with respect to protected attributes $\mathbf{p}$:
\begin{equation}
\label{eqn:criterion}
    p(\ell_m(\mathbf{x} ; \theta_m^\ell) = \mathbf{h} | \mathbf{p}) = p(\ell_m(\mathbf{x}; \theta_m^\ell) = \mathbf{h} | \mathbf{p}^\prime )
\end{equation}
for all $\mathbf{p},\mathbf{p}^\prime \in \mathcal{P}$ and outputs $\mathbf{h} \in {\cal H}$ of $\ell_m(\ \cdot \ ; \theta_m^\ell)$, thereby implying that $\ell_m(\mathbf{x}; \theta_m^\ell)$ and $\mathbf{p}$ are independent.~\cite{NIPS2017_6699} showed that we can use adversarial networks in order to constrain model $\ell_m$ to satisfy Equation~(\ref{eqn:criterion}). $\ell_m$ is pit against an auxiliary adversarial model $a_m = p_{\theta_m^a}(\mathbf{p} | f(\mathbf{x};\theta_m^\ell) = \mathbf{h})$ with parameters $\theta_m^a$ and loss ${\cal L}_m^a(\theta_m^\ell, \theta_m^a)$. The adversarial network $a_m$ is trained to predict $\mathbf{p}$ as much as possible given local representations $\mathbf{h}$. If $p(\ell_m(\mathbf{x}; \theta_m^\ell) = \mathbf{h} | \mathbf{p})$ varies with $\mathbf{p}$, then the corresponding correlation can be captured by adversary $a_m$. On the other hand, if $p(\ell_m(\mathbf{x}; \theta_m^\ell) = \mathbf{h} | \mathbf{p})$ is indeed invariant with respect to $\mathbf{p}$, then adversary $a_m$ should perform randomly. Therefore, we train $\ell_m$ to both minimize the global prediction loss $\mathcal{L}_m^g (\theta_m^\ell, \theta_m^a)$ and to maximize the adversarial loss ${\cal L}_m^a (\theta_m^\ell, \theta_m^a)$. In practice, the local model $\ell_m$, (local copy of the) global model $g$, and adversarial model $a_m$ are updated by solving for the minimax solution:
{\fontsize{9.5}{12}\selectfont
\begin{equation}
\label{eqn:min_thetaf}
    \smash{\hat\theta_m^\ell, \hat\theta_m^g, \hat\theta_m^a} = \arg \min_{\{\theta_m^\ell, \theta_m^g\}} \max_{\theta_m^a} \left[ {\cal L}_m^g (\theta_m^\ell, \theta_m^g) - {\cal L}_m^a (\theta_m^\ell, \theta_m^a) \right].
\end{equation}
}${\cal L}_m^g$ and ${\cal L}_m^a$ are computed using the expected value of the log-likelihood through inference networks $\ell_m$, $g$, and $a_m$. We optimize equation~(\ref{eqn:min_thetaf}) by treating it as a coordinate descent problem and alternately solving for $\smash{\hat\theta_m^\ell, \hat\theta_m^g, \hat\theta_m^a}$ using gradient-based methods (details in Appendix~\ref{sec:theory}). Proposition~\ref{tttt_supp} in Appendix~\ref{sec:theory} further shows that adversarial training learns an optimal local model $\ell_m$ that is invariant with respect to $\mathbf{p}$ under local device data distribution $p(X_m,Y_m,P_m)$. To the best of our knowledge, we are the first to extend this analysis, both theoretically and empirically, to the federated learning setting. Key to this analysis is the separation of local and global models which allows learning of \textit{fair intermediate representations} $\mathbf{h}$.

In Figure~\ref{local}, we also illustrate several other choices of local representation learning using auxiliary local models $a_m$ for (b) unsupervised autoencoding training, where $a_m$ reconstructs $\mathbf{x}$ given $\mathbf{h}$, and (c) self-supervised learning (e.g. jigsaw solving~\cite{DBLP:journals/corr/NorooziF16}), where $a_m$ predicts auxiliary features $\mathbf{z}$ given $\mathbf{h}$. However, when using local supervised learning, $a_m$ and $g$ are similar classification branches (Figure~\ref{local} (a)) both supervised by the target labels. \ours\ can therefore be trained without $a_m$ to directly learn local representations $\mathbf{h}$ for the global model $g$ to make predictions (equation~\ref{joint_obj}). Thus, \ours\ for supervised learning does not incur additional computational complexity while reducing communicated global parameters.

\vspace{-1mm}
\subsection{Theoretical Analysis of Local Fair Representation Learning}
\label{sec:fair_theory}
\vspace{-1mm}

In this section we provide some details of the theoretical proofs and implementation of fair representation learning in our federated learning framework. The setting is adapted from~\cite{NIPS2017_6699}. First recall our dual objective across the local model $\ell_m$, (local copy of the) global model $g$ and auxiliary adversarial model $a_m$:
\begin{equation}
    E(\theta_m^\ell, \theta_m^g, \theta_m^a) = {\cal L}_m^g (\theta_m^\ell, \theta_m^g) - {\cal L}_m^a (\theta_m^\ell, \theta_m^a).
\end{equation}

This implies that the global models should be trained in an adversarial manner since the path of inference when training the (local copy) of the global model also involves the local representation $\mathbf{h}$ and protected attributes $\mathbf{p}$ (refer to Figure~\ref{fig:adv}). When training the global model we optimize for the dual objective across the local model $\ell_m$, global model $g$, and adversarial model $a_m$ by finding the minimax solution $\smash{\hat\theta_m^\ell, \hat\theta_m^g, \hat\theta_m^g}$, defined as
\begin{equation}
\label{eqn:min_thetaf1}
    \smash{\hat\theta_m^\ell, \hat\theta_m^g, \hat\theta_m^a} = \arg \min_{\{\theta_m^\ell, \theta_m^g\}} \max_{\theta_m^a} E(\theta_m^\ell, \theta_m^g, \theta_m^a).
\end{equation}
To do so, we can iteratively solve for $\smash{\hat\theta_m^\ell, \hat\theta_m^g, \hat\theta_m^a}$ in an alternating fashion. In other words, initialize $\smash{\hat\theta^{\ell(0)}_m, \hat\theta^{g(0)}_m, \hat\theta^{a(0)}_m}$ and repeat until convergence:
\begin{align}
    \hat\theta^{\ell(t+1)}_m &= \argmin_{\theta_m^\ell} E(\theta_m^\ell, \theta^{g(t)}_m, \theta^{a(t)}_m), \\
    \hat\theta^{g(t+1)}_m &= \argmin_{\theta_m^g} E(\theta_m^{\ell(t+1)}, \theta_m^g, \theta^{a(t)}_m), \\
    \hat\theta^{a(t+1)}_m &= \argmax_{\theta_m^a} E(\theta_m^{\ell(t+1)}, \theta^{g(t+1)}_m, \theta_m^a).
\end{align}
${\cal L}_m^g$ and ${\cal L}_m^a$ are computed using the expected value of the log likelihood through the inference networks $\ell_m$, $g$, and $a_m$ and the optimization procedure involves using gradient descent and iteratively solving for $\smash{\hat\theta_m^\ell, \hat\theta_m^g, \hat\theta_m^a}$ until convergence. Suppose we define the local data distribution $p(X_m,Y_m,P_m)$, then Proposition~\ref{tttt_supp} shows that this adversarial training procedure learns an optimal local model $\ell_m$ that is at the same time pivotal (invariant) with respect to $\mathbf{p}$ under $p(X_m,Y_m,P_m)$.

\begin{proposition}[Optimality of $\ell_m$, adapted from Proposition 1 in~\cite{NIPS2017_6699}] Suppose we compute losses ${\cal L}_m^g$ and ${\cal L}_m^a$ using the expected log likelihood through the inference networks $\ell_m$, $g$, and $a_m$,
\begin{align}
    {\cal L}_m^g (\theta_m^\ell, \theta_m^g) &= \mathbb{E}_{\mathbf{x} \sim X_m} \mathbb{E}_{\mathbf{y} \sim Y_m|\mathbf{x}} \left[ -\log \left( \sum_{\mathbf{h}} p_{\theta_m^g} (\mathbf{y}|\mathbf{h}) \ p_{\theta_m^\ell} (\mathbf{h}|\mathbf{x}) \right) \right], \\
    {\cal L}_m^a (\theta_m^\ell, \theta_m^a) &= \mathbb{E}_{\mathbf{h} \sim f(X_m;\theta_m^\ell)}  \mathbb{E}_{\mathbf{p} \sim P_m|\mathbf{h}} [-\log p_{\theta_m^a} (\mathbf{p}|\mathbf{h})].
\end{align}
Then, if there is a minimax solution $(\hat{\theta}_m^\ell, \hat{\theta}_m^g, \hat{\theta}_m^a)$ for equation~(\ref{eqn:min_thetaf1}) such that $E(\hat{\theta}_m^\ell, \hat{\theta}_m^g, \hat{\theta}_m^a) = H(Y_m|X_m) - H(P_m)$, then $\ell_m(\ \cdot \ ; \hat{\theta}_m^\ell)$ is both an optimal classifier and a pivotal quantity.
\label{tttt_supp}
\end{proposition}

\begin{proof}
For fixed $\theta_m^\ell$, the adversary $a_m$ is optimal at
\begin{equation}
    {\hat\theta}_{\mathbf{r}_m} = \arg \max_{\theta_m^a} E(\theta_m^\ell, \theta_m^g, \theta_m^a)  = \arg \min_{\theta_m^a} {\cal L}_m^a(\theta_m^\ell, \theta_m^a),
\end{equation}
in which case $p_{{\hat\theta}_{a_m}}(\mathbf{p}|f(X_m;\theta_m^\ell) = \mathbf{h}) =
p(\mathbf{p}|\ell_m (X_m;\theta_m^\ell)=\mathbf{h})$ for all $\mathbf{p}$ and all $\mathbf{h}$, and ${\cal L}_m^a$ reduces to the expected entropy
$\mathbb{E}_{\mathbf{h} \sim \ell_m (X_m;\theta_m^\ell)} \left[ H \left( {P_m|\ell_m (X_m;\theta_m^\ell)=\mathbf{h}} \right) \right]$ of the conditional distribution of the protected variables $\mathbf{p}$.

This expectation corresponds to the conditional entropy of the random variables $P_m$ and $f(X_m;\theta_m^\ell)$ and can be written as $H(P_m|\ell_m(X_m;\theta_m^\ell))$. Accordingly, the value function $E$ can be restated as a function depending only on $\theta_m^\ell$ and $\theta_m^g$:
\begin{equation}
    E'(\theta_m^\ell, \theta_m^g) = {\cal L}_m^g (\theta_m^\ell, \theta_m^g) - H({P_m|\ell_m(X_m;\theta_m^\ell)}).
\end{equation}
By our choice of the objective function we know that
\begin{align}
    {\cal L}_m^g (\theta_m^\ell, \theta_m^g) = \mathbb{E}_{\mathbf{x} \sim X_m} \mathbb{E}_{\mathbf{y} \sim Y_m|\mathbf{x}} \left[ -\log  p (\mathbf{y}|\mathbf{x}) \right] \ge H({Y_m|X_m})
\end{align}
which implies that we have the lower bound
\begin{equation}
    H({Y_m|X_m}) - H(P_m) \leq {\cal L}_m^g (\theta_m^\ell, \theta_m^g) - H({P_m|\ell_m(X_m;\theta_m^\ell)})
\end{equation}
where the equality holds at $\smash{\hat\theta}_m^\ell, \smash{\hat\theta}_m^g = \arg \min_{\{\theta_m^\ell,\theta_m^g\}} E'(\theta_m^\ell, \theta_m^g)$ when:
\begin{enumerate}
    \item $\mathbb{E}_{\mathbf{x} \sim X_m} \mathbb{E}_{\mathbf{y} \sim Y_m|\mathbf{x}} \left[ -\log  p (\mathbf{y}|\mathbf{x}) \right] \ge H({Y_m|X_m})$, which implies that $\smash{\hat\theta}_m^\ell$ and $\smash{\hat\theta}_m^g$ perfectly minimize the negative log-likelihood of $Y_m|X_m$ under $\ell_m$, which happens when $\smash{\hat\theta}_m^\ell$ and $\smash{\hat\theta}_m^g$ are the parameters of an optimal classifier from $X_m$ to $Y_m$ (through an intermediate representation $H$). In this case, ${\cal L}_m^g$ reduces to its minimum value $H({Y_m|X_m})$.
    \item $\smash{\hat\theta}_m^\ell$ maximizes the conditional entropy $H({P_m | \ell_m(X_m;\theta_m^\ell)})$, since $H(P_m | \ell_m(X_m;\theta_m^\ell)) \leq H(P_m)$ from the properties of entropy.
\end{enumerate}
By assumption, the lower bound is active which implies that $H(P_m|\ell_m(X_m;\theta_m^\ell)) = H(P_m)$ because of the second condition. This in turn implies that $P_m$ and $\ell_m (X_m;\theta_m^\ell)$ are independent variables by the properties of (conditional) entropy. Therefore, the optimal classifier $\ell_m(\ \cdot \ ;\smash{\hat\theta}_m^\ell)$ is also a pivotal quantity with respect to the protected attributes $\mathbf{p}$ under local data distribution $p(X_m,Y_m,P_m)$.
\end{proof}

In practice, we optimize for the following dual objectives over local models $\ell_m$, (local copy of the) global model $g$, and adversarial model $a_m$ respectively:
\begin{equation}
    E(\theta_m^\ell, \theta_m^g, \theta_m^a) = {\cal L}_m^g(\theta_m^\ell,\theta_m^g) - \lambda {\cal L}_m^a(\theta_m^\ell, \theta_m^a).
\end{equation}
where $\lambda$ is a hyperparameter that controls the tradeoff between the prediction model and the adversary model.

\begin{figure}[tbp]
\centering
\includegraphics[width=0.4\linewidth]{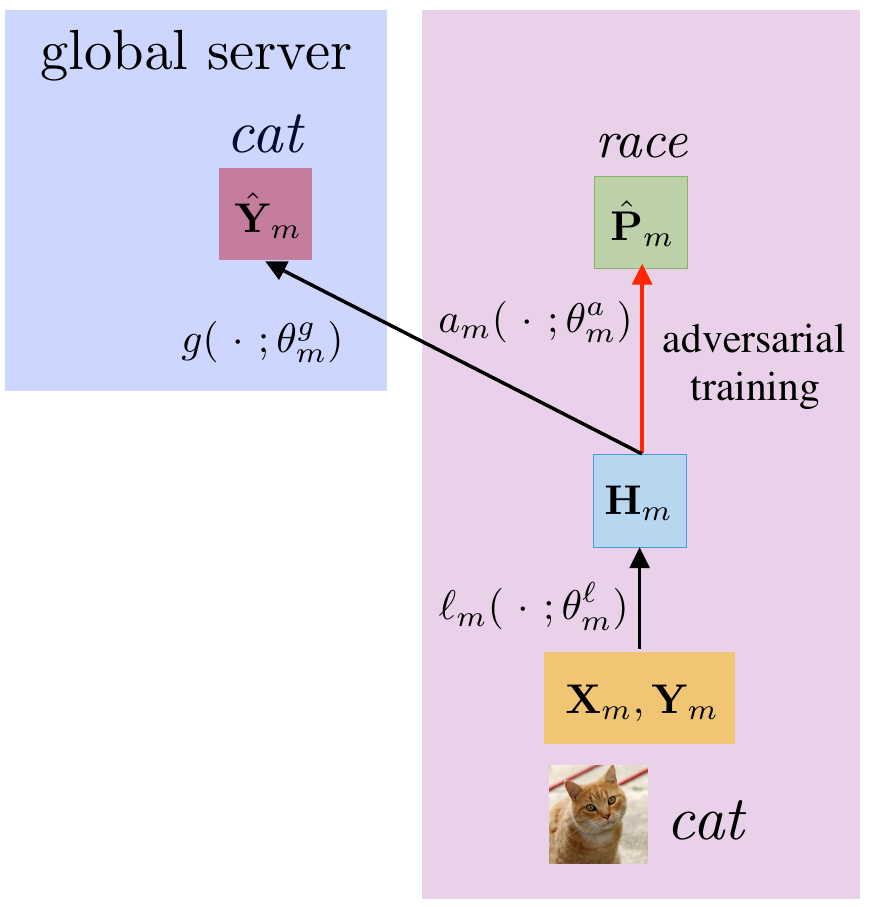}
\caption{A closer look at the inference paths involved in adversarial training. The local models $\ell_m$, (local copy of the) global model $g$ and adversarial model $a_m$ are trained jointly for the global prediction objective and adversarial objective. Refer to equation~(\ref{eqn:min_thetaf1}) for the dual optimization objective over local and global model and adversary parameters respectively.}
\label{fig:adv}
\end{figure}

\vspace{-1mm}
\section{Experimental Details and Extra Results}
\label{details_supp}
\vspace{-1mm}

Here we provide all the details regarding experimental setup, dataset preprocessing, model architectures, model training, and performance evaluation. Our anonymized code is attached in the supplementary material. All experiments are conducted on a single machine with 4 GeForce GTX TITAN X GPUs.

\textbf{A note on hyperparameters used:} For MNIST and CIFAR experiments, we would like to emphasize that our initial set of hyperparameters were taken directly from the default set of hyperparamters in \url{https://github.com/shaoxiongji/federated-learning} for fair comparison across all baselines and our approach. They were \textbf{NOT} manually tuned for \ours\ to perform better.

For synthetic data, there are no hyperparamters involved.

For experiments on mobile data, since it is a new dataset, we start by training the best vanilla federated learning model using \fed\ and use the exact same set of hyperparameters for \ours.

For experiments on fairness, we again use all the default hyperparameters as obtained from the tutorial \url{https://blog.godatadriven.com/fairness-in-ml} and associated code \url{https://github.com/equialgo/fairness-in-ml}.

\vspace{-1mm}
\subsection{Synthetic Data}
\label{synthetic_supp}
\vspace{-1mm}

We set $d = 20, M = 100$, number of train samples per device as $2000$ and the number of test samples per device as $1000$. Data on device $m$ is generated by $\mathbf{x} \sim \mathcal{U}[-1.0, 1.0]$ and teacher weights $\mathbf{u}_m = \mathbf{v} + \mathbf{r}_m$ are sampled as $\mathbf{v} \sim \mathcal{U}[0.0, 1.0]$, $\mathbf{r}_m \sim \mathcal{N}(\mathbf{0}_d, \rho^2 \mathbf{I}_d)$, where $\rho^2$ represents device variance. Labels are observed with noise, $y = \mathbf{u}_m^\top \mathbf{x} + \epsilon$, $\epsilon \sim \mathcal{N}(0, \sigma^2)$, where $\sigma^2$ represents data variance. We plot the average test error when local models perform better due to higher device variance (Figure~\ref{test_plot} left, $\sigma=1.5,\rho=0.1$) and when global models perform better due to lower device variance (Figure~\ref{test_plot} right, $\sigma=1.5,\rho=0.06$). For both settings, using an interpolation of local and global models performs better than either extremes, which supports our analysis.

We also provide several other results demonstrating the effects of data and device variance on the performance local and/or global models. In particular, we fix data variance $\sigma=1.5$ and gradually decrease device variance $\rho \in \{0.5, 0.1, 0.06, 0.02\}$. This results in 4 cases: 1) when local models perform close to optimal (Figure~\ref{test_plot_supp} far left, $\sigma=1.5,\rho=0.5$), 2) when local models perform better (Figure~\ref{test_plot_supp} middle left, $\sigma=1.5,\rho=0.1$), 3) when global models perform better (Figure~\ref{test_plot_supp} middle right, $\sigma=1.5,\rho=0.06$), and 4) when global models perform close to optimal (Figure~\ref{test_plot_supp} far right, $\sigma=1.5,\rho=0.02$). For all settings, using an $\alpha$-interpolation of both local and global models performs either close to the optimal extremes (cases 1 and 4) or better than either extremes (cases 2 and 3).

\begin{figure}[t]
    \centering
    \includegraphics[width=\linewidth]{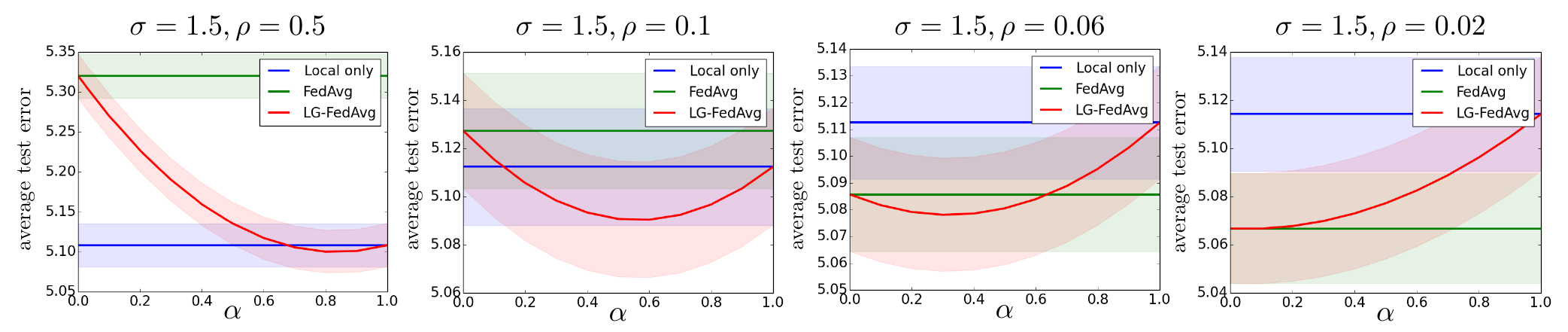}
    \caption{Average test error under four settings: 1) when local models perform close to optimal (far left, $\sigma=1.5,\rho=0.5$), 2) when local models perform better (middle left, $\sigma=1.5,\rho=0.1$), 3) when global models perform better (middle right, $\sigma=1.5,\rho=0.06$), and 4) when global models perform close to optimal (far right, $\sigma=1.5,\rho=0.02$). For all settings, using an $\alpha$-interpolation of both local and global models performs either close to the optimal extremes (cases 1 and 4) or better than either extremes (cases 2 and 3).\vspace{-6mm}}
    \label{test_plot_supp}
\end{figure}

\vspace{-1mm}
\subsection{Model Performance and Communication Efficiency}
\label{efficient_supp}
\vspace{-1mm}

\subsubsection{MNIST}

\textbf{Details:} In all our experiments, we train with number of local epochs $E=1$ and local minibatch size $B=10$. We set $C=0.1$. Images were normalized prior to training and testing. In our experiments, we take the last two layers to form our global model, reducing the number of parameters to $15.79\%$ ($99,978/633,226$). Table~\ref{mnist:config} shows the of hyperparameters used. The dataset can be found here: \url{http://yann.lecun.com/exdb/mnist/}. We train \ours\ with global updates until we reach a goal accuracy ($97.5\%$ for MNIST) before training for additional rounds to jointly update local and global models. Our results are averaged over 10 runs. \# FedAvg and LG Rounds are rounded to the nearest multiple of 5, which we use to calculate the number of parameters communicated. Standard deviations are also reported.

\begin{table}[t]
\fontsize{8.5}{11}\selectfont
\centering
\caption{Table of hyperparameters for MNIST experiments.}
\setlength\tabcolsep{3.5pt}
\begin{tabular}{l | c | c}
\Xhline{3\arrayrulewidth}
Model & Parameter & Value \\
\Xhline{0.5\arrayrulewidth}
\multirow{10}{*}{\fed} 
& Input dim & 784 \\
& Layers & [512, 256, 256, 128] \\
& Output dim & 10 \\
& Loss & cross entropy \\
& Batchsize & 10 \\
& Activation & ReLU \\
& Optimizer & SGD \\
& Learning rate & 0.05 \\
& Momentum & 0.5 \\
& Global epochs & 800 \\
\Xhline{0.5\arrayrulewidth}
\multirow{10}{*}{\textsc{Local Only}}
& Input dim & 784 \\
& Layers & [512, 256, 256, 128] \\
& Output dim & 10 \\
& Loss & cross entropy \\
& Batchsize & 10 \\
& Activation & ReLU \\
& Optimizer & SGD \\
& Learning rate & 0.05 \\
& Momentum & 0.5 \\
& Global epochs & 200 \\
\Xhline{0.5\arrayrulewidth}
\multirow{10}{*}{\ours, {Local model}}
& Input dim & 784 \\
& Layers & [512, 256, 256, 128] \\
& Output dim & 10 \\
& Loss & cross entropy \\
& Batchsize & 10 \\
& Activation & ReLU \\
& Optimizer & SGD \\
& Learning rate & 0.05 \\
& Momentum & 0.5 \\
& Global epochs & 100 \\
\Xhline{0.5\arrayrulewidth}
\multirow{11}{*}{\ours, {Global model}}
& Layers kept & 2 \\
& Input dim & 256 \\
& Layers & [128] \\
& Output dim & 10 \\
& Loss & cross entropy \\
& Batchsize & 10 \\
& Activation & ReLU \\
& Optimizer & SGD \\
& Learning rate & 0.05 \\
& Momentum & 0.5 \\
& Global epochs & 400 \\
\Xhline{3\arrayrulewidth}
\end{tabular}
\label{mnist:config}
\end{table}

\textbf{Extra results:} In this section we provide results on MNIST in comparison with the baselines, see Table~\ref{mnist_supp}. Although MNIST is a slightly smaller dataset, we find that both local and global models help in maintaining performance while using fewer communication parameters.

\definecolor{gg}{RGB}{15,125,15}
\definecolor{rr}{RGB}{190,45,45}

\begin{table*}[!tbp]
\fontsize{8.5}{11}\selectfont
\centering
\caption{\small Comparison of federated learning methods on MNIST with non-iid splits over devices. We report accuracy under both local test and new test settings as well as the total number of parameters communicated across training iterations. Best results in \textbf{bold}. \ours \ outperforms \fed \ under local test and achieves similar performance under new test while using around $50\%$ of the total communicated parameters, across different device splits ($2-10$ classes per device). Mean and standard deviation are computed over $10$ runs.}
\setlength\tabcolsep{1.0pt}

\begin{tabular}{l l || c c c c c}
\Xhline{3\arrayrulewidth}
Data & \multirow{1}{*}{Method} &  \multirow{1}{*}{Local Test Acc. $(\uparrow)$} & \multicolumn{1}{c}{New Test Acc. $(\uparrow)$} & \multirow{1}{*}{FedAvg Rounds} & \multirow{1}{*}{LG Rounds} & \multirow{1}{*}{Params Comm. $(\downarrow)$} \\
\Xhline{0.5\arrayrulewidth}
\multirow{7}{*}{\rotatebox{90}{2 classes/device}} & \fed~\citep{DBLP:journals/corr/McMahanMRA16} & $98.20 \pm 0.05$ & $\mathbf{98.20 \pm 0.05}$ & $800$ & $0$ & $5.6 \times 10^{10}$ \\
& Local only~\citep{DBLP:journals/corr/SmithCST17} & $98.72 \pm 0.35$ & $30.41 \pm 7.88$ & $0$ & $0$ & $0$ \\
& \ours\ (ours) & $\mathbf{98.77 \pm 0.09}$ & $97.72 \pm 0.08$ & $400$ & $100$ & $\mathbf{2.9 \times 10^{10}}$ \\
& \ours\ (ours) & $98.71 \pm 0.08$ & $97.94 \pm 0.06$ & $500$ & $100$ & $3.6 \times 10^{10}$ \\
& \ours\ (ours) & $98.70 \pm 0.01$ & $98.03 \pm 0.05$ & $600$ & $100$ & $4.3 \times 10^{10}$ \\
& \ours\ (ours) & $98.63 \pm 0.09$ & $98.07 \pm 0.03$ & $700$ & $100$ & $5.0 \times 10^{10}$ \\
& \ours\ (ours) & $98.54 \pm 0.05$ & $\mathbf{98.17 \pm 0.05}$ & $800$ & $100$ & $5.7 \times 10^{10}$ \\
\Xhline{3\arrayrulewidth}
\end{tabular}

\vspace{2mm}

\begin{tabular}{l l || c c c c c}
\Xhline{3\arrayrulewidth}
Data & \multirow{1}{*}{Method} &  \multirow{1}{*}{Local Test Acc. $(\uparrow)$} & \multicolumn{1}{c}{New Test Acc. $(\uparrow)$} & \multirow{1}{*}{FedAvg Rounds} & \multirow{1}{*}{LG Rounds} & \multirow{1}{*}{Params Comm. $(\downarrow)$} \\
\Xhline{0.5\arrayrulewidth}
\multirow{7}{*}{\rotatebox{90}{3 classes/device}} & \fed~\citep{DBLP:journals/corr/McMahanMRA16} & $98.20 \pm 0.02$ & $\mathbf{98.20 \pm 0.02}$ & $800$ & $0$ & $5.6 \times 10^{10}$ \\
& Local only~\citep{DBLP:journals/corr/SmithCST17} & $97.55 \pm 0.30$ & $36.11 \pm 10.13$ & $0$ & $0$ & $0$ \\
& \ours\ (ours) & $\mathbf{98.55 \pm 0.09}$ & $97.92 \pm 0.08$ & $400$ & $100$ & $\mathbf{2.9 \times 10^{10}}$ \\
& \ours\ (ours) & $98.38 \pm 0.04$ & $98.03 \pm 0.08$ & $500$ & $100$ & $3.6 \times 10^{10}$ \\
& \ours\ (ours) & $98.44 \pm 0.06$ & $98.10 \pm 0.03$ & $600$ & $100$ & $4.3 \times 10^{10}$ \\
& \ours\ (ours) & $98.37 \pm 0.08$ & $98.14 \pm 0.02$ & $700$ & $100$ & $5.0 \times 10^{10}$ \\
& \ours\ (ours) & $98.34 \pm 0.10$ & $\mathbf{98.18 \pm 0.05}$ & $800$ & $100$ & $5.7 \times 10^{10}$ \\
\Xhline{3\arrayrulewidth}
\end{tabular}

\vspace{2mm}

\begin{tabular}{l l || c c c c c}
\Xhline{3\arrayrulewidth}
Data & \multirow{1}{*}{Method} &  \multirow{1}{*}{Local Test Acc. $(\uparrow)$} & \multicolumn{1}{c}{New Test Acc. $(\uparrow)$} & \multirow{1}{*}{FedAvg Rounds} & \multirow{1}{*}{LG Rounds} & \multirow{1}{*}{Params Comm. $(\downarrow)$} \\
\Xhline{0.5\arrayrulewidth}
\multirow{7}{*}{\rotatebox{90}{4 classes/device}} & \fed~\citep{DBLP:journals/corr/McMahanMRA16} & $98.21 \pm 0.05$ & $\mathbf{98.21 \pm 0.05}$ & $800$ & $0$ & $5.6 \times 10^{10}$ \\
& Local only~\citep{DBLP:journals/corr/SmithCST17} & $96.53 \pm 0.41$ & $43.15 \pm 16.41$ & $0$ & $0$ & $0$ \\
& \ours\ (ours) & $98.32 \pm 0.08$ & $97.98 \pm 0.08$ & $400$ & $100$ & $\mathbf{2.9 \times 10^{10}}$ \\
& \ours\ (ours) & $98.28 \pm 0.07$ & $98.00 \pm 0.05$ & $500$ & $100$ & $3.6 \times 10^{10}$ \\
& \ours\ (ours) & $98.33 \pm 0.05$ & $98.12 \pm 0.05$ & $600$ & $100$ & $4.3 \times 10^{10}$ \\
& \ours\ (ours) & $98.30 \pm 0.06$ & $98.12 \pm 0.03$ & $700$ & $100$ & $5.0 \times 10^{10}$ \\
& \ours\ (ours) & $\mathbf{98.34 \pm 0.06}$ & $\mathbf{98.20 \pm 0.06}$ & $800$ & $100$ & $5.7 \times 10^{10}$ \\
\Xhline{3\arrayrulewidth}
\end{tabular}

\vspace{2mm}

\begin{tabular}{l l || c c c c c}
\Xhline{3\arrayrulewidth}
Data & \multirow{1}{*}{Method} &  \multirow{1}{*}{Local Test Acc. $(\uparrow)$} & \multicolumn{1}{c}{New Test Acc. $(\uparrow)$} & \multirow{1}{*}{FedAvg Rounds} & \multirow{1}{*}{LG Rounds} & \multirow{1}{*}{Params Comm. $(\downarrow)$} \\
\Xhline{0.5\arrayrulewidth}
\multirow{7}{*}{\rotatebox{90}{5 classes/device}} & \fed~\citep{DBLP:journals/corr/McMahanMRA16} & $98.13 \pm 0.05$ & $\mathbf{98.13 \pm 0.05}$ & $800$ & $0$ & $5.6 \times 10^{10}$ \\
& Local only~\citep{DBLP:journals/corr/SmithCST17} & $95.47 \pm 0.31$ & $58.69 \pm 4.11$ & $0$ & $0$ & $0$ \\
& \ours\ (ours) & $98.18 \pm 0.06$ & $97.82 \pm 0.10$ & $400$ & $100$ & $\mathbf{2.9 \times 10^{10}}$ \\
& \ours\ (ours) & $\mathbf{98.26 \pm 0.07}$ & $98.01 \pm 0.07$ & $500$ & $100$ & $3.6 \times 10^{10}$ \\
& \ours\ (ours) & $98.23 \pm 0.04$ & $98.06 \pm 0.05$ & $600$ & $100$ & $4.3 \times 10^{10}$ \\
& \ours\ (ours) & $98.23 \pm 0.04$ & $98.10 \pm 0.05$ & $700$ & $100$ & $5.0 \times 10^{10}$ \\
& \ours\ (ours) & $98.21 \pm 0.07$ & $\mathbf{98.12 \pm 0.07}$ & $800$ & $100$ & $5.7 \times 10^{10}$ \\
\Xhline{3\arrayrulewidth}
\end{tabular}

\vspace{2mm}

\begin{tabular}{l l || c c c c c}
\Xhline{3\arrayrulewidth}
Data & \multirow{1}{*}{Method} &  \multirow{1}{*}{Local Test Acc. $(\uparrow)$} & \multicolumn{1}{c}{New Test Acc. $(\uparrow)$} & \multirow{1}{*}{FedAvg Rounds} & \multirow{1}{*}{LG Rounds} & \multirow{1}{*}{Params Comm. $(\downarrow)$} \\
\Xhline{0.5\arrayrulewidth}
\multirow{7}{*}{\rotatebox{90}{10 classes/device (iid) }} & \fed~\citep{DBLP:journals/corr/McMahanMRA16} &  $97.93 \pm 0.08$ &	$\mathbf{97.93 \pm 0.08}$&	$800$&	$0$&	$5.6 \times 10^{10}$ \\
& Local only~\citep{DBLP:journals/corr/SmithCST17} & $88.03 \pm 0.37$&	$86.24 \pm 0.87$&	$0$&	$0$&	$0$\\
& \ours\ (ours) & $97.59 \pm 0.08$&	$97.61 \pm 0.08$&   $400$&	$100$&	$\mathbf{2.9 \times 10^{10}}$ \\
& \ours\ (ours) & $97.78 \pm 0.13$&	$97.82 \pm 0.14$&	$500$&	$100$&	$3.6 \times 10^{10}$ \\
& \ours\ (ours) & $97.84 \pm 0.10$&	$97.86 \pm 0.08$&	$600$&	$100$&	$4.3 \times 10^{10}$ \\
& \ours\ (ours) & $97.85 \pm 0.09$&	$97.88 \pm 0.09$&	$700$&  $100$&	$5.0 \times 10^{10}$ \\
& \ours\ (ours) & $\mathbf{97.91 \pm 0.10}$ &	$\mathbf{97.93 \pm 0.07}$&	$800$&	$100$&	$5.7 \times 10^{10}$ \\
\Xhline{3\arrayrulewidth}
\end{tabular}

\vspace{2mm}

\label{mnist_supp}
\vspace{-4mm}
\end{table*}

\subsubsection{CIFAR10}

\textbf{Details:} We train with number of local epochs $E=1$ and local minibatch size $B=50$. We set $C=0.1$. Images are randomly cropped to size $32$, randomly flipped horizontally with probability $p=0.5$, resized to $224 \times 224$, and normalized. For our model architecture, we chose Lenet-5. We use the two convolutional layers for the global model in our \ours\ method to minimize the number of parameters. We therefore reduce the number of parameters to $4.48\%$ ($2872/64102$). Table~\ref{mnist:config} shows a table of additional hyperparameters used. The dataset can be found here: \url{https://www.cs.toronto.edu/~kriz/cifar.html}. We train \ours\ with global updates until we reach a goal accuracy ($57\%$ for CIFAR-10) before training for additional rounds to jointly update local and global models. Our results are averaged over 10 runs and we report standard deviations. \# FedAvg and LG Rounds are rounded to the nearest multiple of 5, which we use to calculate the number of parameters communicated.

\begin{table}[t]
\fontsize{8.5}{11}\selectfont
\centering
\caption{Table of hyperparameters for CIFAR-10 experiments.}
\setlength\tabcolsep{3.5pt}
\begin{tabular}{l | c | c}
\Xhline{3\arrayrulewidth}
Model & Parameter & Value \\
\Xhline{0.5\arrayrulewidth}
\multirow{7}{*}{\fed} 
& Loss & cross entropy \\
& Batchsize & 50 \\
& Optimizer & SGD \\
& Learning rate & 0.1 \\
& Momentum & 0.5 \\
& Learning rate decay & 0.005 \\
& Global epochs & 1800 \\
\Xhline{0.5\arrayrulewidth}
\multirow{7}{*}{\textsc{Local Only}}
& Loss & cross entropy \\
& Batchsize & 50 \\
& Optimizer & SGD \\
& Learning rate & 0.1 \\
& Momentum & 0.5 \\
& Learning rate decay & 0.005 \\
& Global epochs & 200 \\
\Xhline{0.5\arrayrulewidth}
\multirow{7}{*}{\ours, {Local model}}
& Loss & cross entropy \\
& Batchsize & 50 \\
& Optimizer & SGD \\
& Learning rate & 0.1 \\
& Momentum & 0.5 \\
& Learning rate decay & 0.005 \\
& Global epochs & 100 \\
\Xhline{0.5\arrayrulewidth}
\multirow{7}{*}{\ours, {Global model}}
& Loss & cross entropy \\
& Batchsize & 50 \\
& Optimizer & SGD \\
& Learning rate & 0.1 \\
& Momentum & 0.5 \\
& Learning rate decay & 0.005 \\
& Global epochs & 1200 \\
\Xhline{3\arrayrulewidth}
\end{tabular}
\label{cifar10:config}
\end{table}

\textbf{Extra results:} In this section we provide more results and also show a sensitivity analysis to various hyperparameters especially regarding the data splits across devices and the local-global model split in our method. See Table~\ref{cifar10_supp}. Our results are especially strong here: across different data splits (different number of users per device), \ours\ consistently performs better on local test and new test while using fewer parameters.

\definecolor{gg}{RGB}{15,125,15}
\definecolor{rr}{RGB}{190,45,45}

\begin{table*}[!tbp]
\fontsize{8.5}{11}\selectfont
\centering
\caption{\small Comparison of federated learning methods on CIFAR-10 with non-iid split over devices. We report accuracy under both local test and new test settings as well as the total number of parameters communicated across training iterations. Best results in \textbf{bold}. \ours \ outperforms \fed\ and under local test and achieves similar performance under new test while using around $50\%$ of the total communicated parameters, across different device splits ($2-10$ classes per device). Mean and standard deviation are computed over $10$ runs.}
\setlength\tabcolsep{1.0pt}

\begin{tabular}{l l || c c c c c}
\Xhline{3\arrayrulewidth}
Data & \multirow{1}{*}{Method} &  \multirow{1}{*}{Local Test Acc. $(\uparrow)$} & \multicolumn{1}{c}{New Test Acc. $(\uparrow)$} & \multirow{1}{*}{FedAvg Rounds} & \multirow{1}{*}{LG Rounds} & \multirow{1}{*}{Params Comm. $(\downarrow)$} \\
\Xhline{0.5\arrayrulewidth}
\multirow{7}{*}{\rotatebox{90}{2 classes/device}} & \fed~\citep{DBLP:journals/corr/McMahanMRA16} & $58.99 \pm 1.50$ & $58.99 \pm 1.50$ & $1800$ & $0$ & $12.7 \times 10^{9}$ \\
& Local only~\citep{DBLP:journals/corr/SmithCST17} &$87.93 \pm 2.14$ & $10.03 \pm 0.06$ & $0$ & $0$ & $0$ \\
& \ours\ (ours) & $90.20 \pm 0.79$ & $56.52 \pm 1.59$ & $1000$ & $100$ & $\mathbf{7.1 \times 10^{9}}$ \\
& \ours\ (ours) & $90.77 \pm 0.50$ & $57.95 \pm 1.48$ & $1200$ & $100$ & $8.5 \times 10^{9}$ \\
& \ours\ (ours) & $91.07 \pm 0.62$ & $59.28 \pm 1.70$ & $1400$ & $100$ & $9.9 \times 10^{9}$ \\
& \ours\ (ours) & $91.45 \pm 0.77$ & $59.96 \pm 1.61$ & $1600$ & $100$ & $11.3 \times 10^{9}$ \\
& \ours\ (ours) & $\mathbf{91.77 \pm 0.56}$ & $\mathbf{60.79 \pm 1.45}$ & $1800$ & $100$ & $12.7 \times 10^{9}$ \\
\Xhline{3\arrayrulewidth}
\end{tabular}

\vspace{2mm}

\begin{tabular}{l l || c c c c c}
\Xhline{3\arrayrulewidth}
Data & \multirow{1}{*}{Method} &  \multirow{1}{*}{Local Test Acc. $(\uparrow)$} & \multicolumn{1}{c}{New Test Acc. $(\uparrow)$} & \multirow{1}{*}{FedAvg Rounds} & \multirow{1}{*}{LG Rounds} & \multirow{1}{*}{Params Comm. $(\downarrow)$} \\
\Xhline{0.5\arrayrulewidth}
\multirow{7}{*}{\rotatebox{90}{3 classes/device}} & \fed~\citep{DBLP:journals/corr/McMahanMRA16} & $63.68 \pm 0.35$ & $63.68 \pm 0.350$ & $1800$ & $0$ & $12.7 \times 10^{9}$ \\
& Local only~\citep{DBLP:journals/corr/SmithCST17} &$79.79 \pm 1.05$ & $10.00 \pm 0.00$ & $0$ & $0$ & $0$ \\
& \ours\ (ours) & $86.01 \pm 0.55$ & $61.78 \pm 0.61$ & $1000$ & $100$ & $\mathbf{7.1 \times 10^{9}}$ \\
& \ours\ (ours) & $85.13 \pm 0.76$ & $63.01 \pm 0.58$ & $1200$ & $100$ & $8.5 \times 10^{9}$ \\
& \ours\ (ours) & $86.69 \pm 0.58$ & $63.57 \pm 0.31$ & $1400$ & $100$ & $9.9 \times 10^{9}$ \\
& \ours\ (ours) & $86.70 \pm 0.49$ & $63.39 \pm 1.68$ & $1600$ & $100$ & $11.3 \times 10^{9}$ \\
& \ours\ (ours) & $\mathbf{87.26 \pm 0.63}$ & $\mathbf{64.79 \pm 0.55}$ & $1800$ & $100$ & $12.7 \times 10^{9}$ \\
\Xhline{3\arrayrulewidth}
\end{tabular}

\vspace{2mm}

\begin{tabular}{l l || c c c c c}
\Xhline{3\arrayrulewidth}
Data & \multirow{1}{*}{Method} &  \multirow{1}{*}{Local Test Acc. $(\uparrow)$} & \multicolumn{1}{c}{New Test Acc. $(\uparrow)$} & \multirow{1}{*}{FedAvg Rounds} & \multirow{1}{*}{LG Rounds} & \multirow{1}{*}{Params Comm. $(\downarrow)$} \\
\Xhline{0.5\arrayrulewidth}
\multirow{7}{*}{\rotatebox{90}{4 classes/device}} & \fed~\citep{DBLP:journals/corr/McMahanMRA16} & $65.54 \pm 0.66$ & $65.54 \pm 0.66$ & $1800$ & $0$ & $12.7 \times 10^{9}$ \\
& Local only~\citep{DBLP:journals/corr/SmithCST17} & $78.01 \pm 0.46$ & $10.22 \pm 0.29$ & $0$ & $0$ & $0$ \\
& \ours\ (ours) & $82.56 \pm 0.54$ & $63.62 \pm 0.64$ & $1000$ & $100$ & $\mathbf{7.1 \times 10^{9}}$ \\
& \ours\ (ours) & $83.02 \pm 0.47$ & $64.40 \pm 0.45$ & $1200$ & $100$ & $8.5 \times 10^{9}$ \\
& \ours\ (ours) & $83.61 \pm 0.26$ & $65.41 \pm 0.71$ & $1400$ & $100$ & $9.9 \times 10^{9}$ \\
& \ours\ (ours) & $83.78 \pm 0.56$ & $65.99 \pm 0.55$ & $1600$ & $100$ & $11.3 \times 10^{9}$ \\
& \ours\ (ours) & $\mathbf{84.14 \pm 0.42}$ & $\mathbf{66.48 \pm 0.74}$ & $1800$ & $100$ & $12.7 \times 10^{9}$ \\
\Xhline{3\arrayrulewidth}
\end{tabular}

\vspace{2mm}

\begin{tabular}{l l || c c c c c}
\Xhline{3\arrayrulewidth}
Data & \multirow{1}{*}{Method} &  \multirow{1}{*}{Local Test Acc. $(\uparrow)$} & \multicolumn{1}{c}{New Test Acc. $(\uparrow)$} & \multirow{1}{*}{FedAvg Rounds} & \multirow{1}{*}{LG Rounds} & \multirow{1}{*}{Params Comm. $(\downarrow)$} \\
\Xhline{0.5\arrayrulewidth}
\multirow{7}{*}{\rotatebox{90}{5 classes/device}} & \fed~\citep{DBLP:journals/corr/McMahanMRA16} & $67.21 \pm 0.45$ & $67.21 \pm 0.45$ & $1800$ & $0$ & $12.7 \times 10^{9}$ \\
& Local only~\citep{DBLP:journals/corr/SmithCST17} & $73.42 \pm 0.56$ & $10.51 \pm 0.49$ & $0$ & $0$ & $0$ \\
& \ours\ (ours) & $80.97 \pm 0.62$ & $65.34 \pm 1.00$ & $1000$ & $100$ & $\mathbf{7.1 \times 10^{9}}$ \\
& \ours\ (ours) & $81.50 \pm 0.52$ & $66.32 \pm 0.48$ & $1200$ & $100$ & $8.5 \times 10^{9}$ \\
& \ours\ (ours) & $81.92 \pm 0.55$ & $67.26 \pm 0.44$ & $1400$ & $100$ & $9.9 \times 10^{9}$ \\
& \ours\ (ours) & $82.29 \pm 0.38$ & $67.61 \pm 0.61$ & $1600$ & $100$ & $11.3 \times 10^{9}$ \\
& \ours\ (ours) & $\mathbf{82.51 \pm 0.50}$ & $\mathbf{68.32 \pm 0.56}$ & $1800$ & $100$ & $12.7 \times 10^{9}$ \\
\Xhline{3\arrayrulewidth}
\end{tabular}

\vspace{2mm}

\begin{tabular}{l l || c c c c c}
\Xhline{3\arrayrulewidth}
Data & \multirow{1}{*}{Method} &  \multirow{1}{*}{Local Test Acc. $(\uparrow)$} & \multicolumn{1}{c}{New Test Acc. $(\uparrow)$} & \multirow{1}{*}{FedAvg Rounds} & \multirow{1}{*}{LG Rounds} & \multirow{1}{*}{Params Comm. $(\downarrow)$} \\
\Xhline{0.5\arrayrulewidth}
\multirow{7}{*}{\rotatebox{90}{10 classes/device (iid)}} & \fed~\citep{DBLP:journals/corr/McMahanMRA16} & $67.74 \pm 0.45$ & $67.74 \pm 0.45$ & $1800$ & $0$ & $12.7 \times 10^{9}$ \\
& Local only~\citep{DBLP:journals/corr/SmithCST17} & $45.54 \pm 0.31$ & $16.90 \pm 3.09$ & $0$ & $0$ & $0$ \\
& \ours\ (ours) & $68.09 \pm 0.66$ & $67.93 \pm 0.61$ & $1000$ & $100$ & $\mathbf{7.1 \times 10^{9}}$ \\
& \ours\ (ours) & $68.97 \pm 0.55$ & $68.90 \pm 0.54$ & $1200$ & $100$ & $8.5 \times 10^{9}$ \\
& \ours\ (ours) & $69.36 \pm 0.37$ & $69.16 \pm 0.30$ & $1400$ & $100$ & $9.9 \times 10^{9}$ \\
& \ours\ (ours) & $69.64 \pm 0.38$ & $69.52 \pm 0.44$ & $1600$ & $100$ & $11.3 \times 10^{9}$ \\
& \ours\ (ours) & $\mathbf{69.89 \pm 0.48}$ & $\mathbf{69.76 \pm 0.49}$ & $1800$ & $100$ & $12.7 \times 10^{9}$ \\
\Xhline{3\arrayrulewidth}
\end{tabular}

\vspace{2mm}

\label{cifar10_supp}
\vspace{-4mm}
\end{table*}

\subsubsection{VQA}
\label{vqa_supp}

\textbf{Details:} We adapt the baseline model from \citep{VQA} without \textit{norm I} image channel embeddings. We also substitute the VGGNet~\citep{Simonyan14c} used in the original baseline model with a pre-trained ResNet-18~\citep{DBLP:journals/corr/HeZRS15}. Finally we use the deep LSTM~\citep{hochreiter1997long} embedding, which is an LSTM that consists of two hidden layers. For the \ours\ method, the global model uses the two final fully connected layers of the image and question channels, as well as the the additional two fully connected layers following the fusion via element-wise multiplication. The global model reduces the number of parameters to $9.53\%$ ($5149200/54042572$). We use 50 devices and set number of local epochs $E=1$, local minibatch size $B=100$, fraction of devices sampled per round $C=0.1$. To train and evaluate our models, we use the data from the following: \url{https://visualqa.org/download.html}. Table~\ref{vqa:config} shows a table of hyperparameters used, which strictly follows the baseline model from \citep{VQA}. We first trained the best \fed\ model and used the exact same hyperparamters to train \ours\ as well.

\begin{figure}[t]
    \centering
    \includegraphics[width=0.5\linewidth]{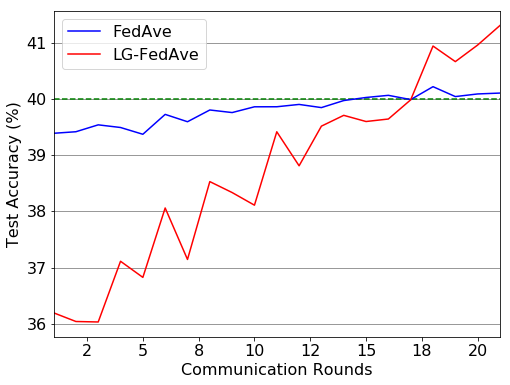}
    \caption{Test accuracy on VQA across $20$ rounds (dotted green line marks the goal accuracy of $40\%$ used in Table~\ref{vqa}). \ours\ reaches an accuracy of $41.30\%$ compared to $40.22\%$ for \fed\ while using only $9.53\%$ of the parameters.\vspace{-4mm}}
    \label{vqa_graph}
\end{figure}

\textbf{Extra results:} In Figure \ref{vqa_graph}, we plot the convergence of test accuracy across communication rounds. \ours\ outperforms \fed\ after $20$ rounds while requiring only $9.53\%$ of parameters in \fed \ and continues to improve.

\subsubsection{Other Baselines}

The \textsc{MTL} baseline~\citep{DBLP:journals/corr/SmithCST17} is implemented by training local models each with parameters $\mathbf{w}_m$ and adding a joint regularization term $\mathcal{R}(\mathbf{W}, \mathbf{\Omega}) = \lambda_1 \textrm{tr} (\mathbf{W} \mathbf{\Omega} \mathbf{W}^\top) + \lambda_2 \| \mathbf{W} \|_F^2$ where $\lambda_1$ and $\lambda_2$ are hyperparmeters, and $\mathbf{W} = [\mathbf{w}_1, . . . , \mathbf{w}_M]$ is a matrix whose $m$-th column is the weight vector for the $m$-th device. We choose $\mathbf{\Omega} = (\mathbf{I}_M - \frac{1}{M} \mathbf{1} \mathbf{1}^\top)^2$ which implements mean-regularized multitask learning~\citep{10.1145/1014052.1014067}, which assumes that all the tasks form one cluster, and that the weight vectors are close to their mean. For CIFAR-10, computing the full $\mathbf{W}$ requires storing a matrix of size $p \times M$. Optimizing and storing this matrix becomes infeasible as the number of users or model size increases. Even for our experimental setting, where $M = 100$ and $p \approx 64,000$ (number of parameters in each local model), running MTL require the following relaxations. First, we reduce $M$ to $10$, reducing the size of $\mathbf{W}$ and therefore the number of parameters communicated per round.

\begin{table}[t]
\fontsize{8.5}{11}\selectfont
\centering
\caption{Table of hyperparameters for VQA experiments.}
\setlength\tabcolsep{3.5pt}
\begin{tabular}{l | c | c}
\Xhline{3\arrayrulewidth}
Model & Parameter & Value \\
\Xhline{0.5\arrayrulewidth}
\multirow{7}{*}{\fed} 
& Loss & cross entropy \\
& Batchsize & 100 \\
& Optimizer & SGD \\
& Learning rate & 0.01 \\
& Momentum & 0.9 \\
& Learning rate decay & 0.0005 \\
& Global epochs & 100 \\
\Xhline{0.5\arrayrulewidth}
\multirow{7}{*}{\ours, {Local model}}
& Loss & cross entropy \\
& Batchsize & 100 \\
& Optimizer & SGD \\
& Learning rate & 0.01 \\
& Momentum & 0.9 \\
& Learning rate decay & 0.0005 \\
& Global epochs & 100 \\
\Xhline{0.5\arrayrulewidth}
\multirow{7}{*}{\ours, {Global model}}
& Loss & cross entropy \\
& Batchsize & 100 \\
& Optimizer & SGD \\
& Learning rate & 0.01 \\
& Momentum & 0.9 \\
& Learning rate decay & 0.0005 \\
& Global epochs & 100 \\
\Xhline{3\arrayrulewidth}
\end{tabular}
\label{vqa:config}
\end{table}

\vspace{-1mm}
\subsection{Learning Personalized Mood Predictors from Mobile Data}
\label{mood_supp}
\vspace{-1mm}

\textbf{Dataset details:} We designed and collected a new dataset called Mobile Assessment for the Prediction of Suicide (MAPS). MAPS was designed to elucidate real-time indicators of suicide risk in adolescents ages $13-18$ years. Current adolescent suicide ideators and recent suicide attempters along with aged-matched psychiatric controls with no lifetime suicidal thoughts and behaviors completed baseline clinical assessments (i.e., lifetime mental disorders, current psychiatric symptoms). Following the baseline clinical characterization, a smartphone app, the Effortless Assessment of Risk States (EARS), was installed onto adolescents' phones, and passive sensor data were acquired for $6$-months. Notably, during EARS installation, a keyboard logger is configured on adolescents’ phones, which then tracks all words typed into the phone as well as they app used during this period. Each day during the $6$-month follow-up, participants also were asked to rate their mood on the previous day on a scale ranging from $1-100$, with higher scores indicating a better mood.

\textbf{All users have given consent for their mobile device data to be collected and shared with us for research purposes.}

MAPS is a realistic federated learning benchmark since it contains real-world data with privacy concerns and high device variance due to highly personalized use of mobile phones. We used a preliminary preprocessed version containing $572$ samples across $14$ participants. We discretize the scores into $5$ bins for $5$-way classification. We use a random $80/10/10$ split for training/validation/testing, conduct all experiments $10$ times, and report the average accuracy and standard deviation.


\textbf{Model details:} To assess how mobile text data can be used to make personalized mood predictions, we train a MLP classifier on top of a Bi-LSTM encoder. The Bi-LSTM has 128 hidden units and the MLP has two hidden layers, each with size 512. We conduct our experiments over 10 iterations. Within each iteration, we use a random 80/10/10 split for training/validation/testing. We train and validate our model 5 times on this split and select the model that performs best on the validation set. We use the test accuracy of this best-performing model as the test accuracy for the iteration. We report the average accuracy and standard deviation over all 10 iterations in Figure~\ref{test_plot}(d). In addition to local only and \fed\ results, we plot the performance of \ours\ across different splits of local and global models (i.e. $\alpha \in \{ 0.2, 0.4, 0.6, 0.8\}$). Consistent with our theoretical findings, an $\alpha$-split across local and global models leverages both personalized representations per devices as well as statistical strength sharing through data across all devices, outperforming either local or global extremes.

\begin{table}[t]
\fontsize{8.5}{11}\selectfont
\centering
\caption{Table of hyperparameters for MAPS experiments.}
\setlength\tabcolsep{3.5pt}
\begin{tabular}{l | c | c}
\Xhline{3\arrayrulewidth}
Model & Parameter & Value \\
\Xhline{0.5\arrayrulewidth}
\multirow{7}{*}{\fed} 
& BiLSTM encoder hidden units & $128$ \\
& MLP hidden layers & [512, 512] \\
& Loss & cross entropy \\
& Max tokens per batch & $2000$ \\
& Optimizer & adam \\
& Learning rate & 5e-3 \\
& Learning rate shrink & $0.5$ \\
\Xhline{0.5\arrayrulewidth}
\multirow{7}{*}{Local only}
& BiLSTM encoder hidden units & $128$ \\
& MLP hidden layers & [512, 512] \\
& Loss & cross entropy \\
& Max tokens per batch & $2000$ \\
& Optimizer & adam \\
& Learning rate & 5e-3 \\
& Learning rate shrink & $0.5$ \\
\Xhline{0.5\arrayrulewidth}
\multirow{7}{*}{Global}
& BiLSTM encoder hidden units & 128 \\
& MLP hidden layers & [512, 512] \\
& Loss & cross entropy \\
& Max tokens per batch & 2000 \\
& Optimizer & adam \\
& Learning rate & 5e-3 \\
& Learning rate shrink & 0.5 \\
\Xhline{3\arrayrulewidth}
\end{tabular}
\label{vqa:config}
\end{table}

\vspace{-1mm}
\subsection{Heterogeneous Data in an Online Setting}
\label{hetero_supp}
\vspace{-1mm}

\textbf{Details:} Our experiments for the rotated MNIST follow the same settings and hyperparameter selection as our normal MNIST experiments (section~\ref{efficient_supp}). However, we include an additional device, which randomly samples 3000 and 500 images from the train and test sets respectively and rotates them by a fixed 90 degrees. We show some samples of the rotated MNIST images we used in Figure~\ref{mnist_pics}, where the top row shows the normal MNIST images used during training and the bottom row shows the rotated MNIST images on the new test device.

\begin{figure}[tbp]
\centering
\includegraphics[width=0.5\linewidth]{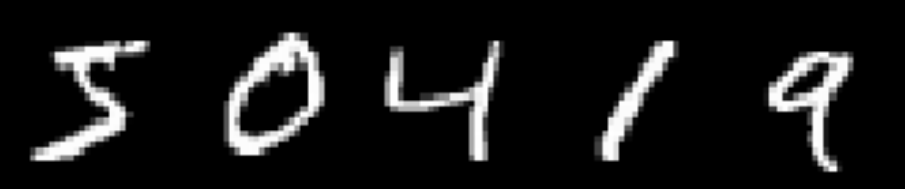}
\includegraphics[width=0.5\linewidth]{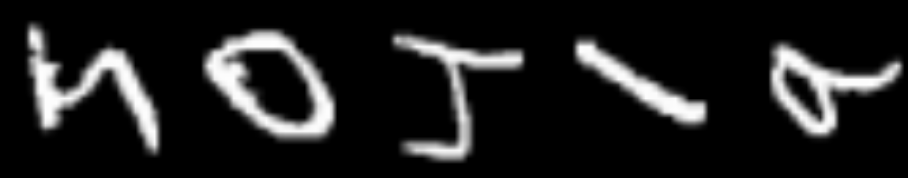}
\caption{Sample MNIST images used for training (top) and their rotated counterparts used to test the impact of heterogeneous data on a trained federated model in an online setting (bottom).}
\label{mnist_pics}
\end{figure}

\vspace{-1mm}
\subsection{Learning Fair Representations}
\label{fair_supp}
\vspace{-1mm}

\textbf{Details:} For method 1, \fed, we train the global model and global adversary for $50$ outer epochs, within which the number of local epochs $E=10$. For methods 2 and 3 involving local models, we begin by pre-training the local models and local adversaries for $10$ epochs before joint local and global training for 10 epochs. Table~\ref{fair:config} shows the table of all hyperparameters used. Experiments were run 10 times with the same hyperparameters but different random seeds. We aimed to keep the local, global, and adversary models as similar as possible between the three baselines for fair comparison. Apart from the number of local and global epochs all hyperparameters were kept the same from the tutorial \url{https://blog.godatadriven.com/fairness-in-ml} and associated code \url{https://github.com/equialgo/fairness-in-ml}. The data can be found at \url{https://archive.ics.uci.edu/ml/datasets/Adult}.

\begin{table}[t]
\fontsize{8.5}{11}\selectfont
\centering
\caption{Table of hyperparameters for experiments on learning fair representations on the UCI adult dataset.}
\setlength\tabcolsep{3.5pt}
\begin{tabular}{l | c | c}
\Xhline{3\arrayrulewidth}
Model & Parameter & Value \\
\Xhline{0.5\arrayrulewidth}
\multirow{11}{*}{\fed, {Global model}} & Input dim & 93 \\
& Layers & [32,32,32] \\
& Output dim & 1 \\
& Loss & cross entropy \\
& Dropout & 0.2 \\
& Batchsize & 32 \\
& Activation & ReLU \\
& Optimizer & SGD \\
& Learning rate & 0.1 \\
& Momentum & 0.5 \\
& Global epochs & 50 \\
\Xhline{0.5\arrayrulewidth}
\multirow{11}{*}{\fed, {Global Adversary}} & Input dim & 32 \\
& Layers & [32,32,32] \\
& Output dim & 2 \\
& Loss & cross entropy \\
& Dropout & 0.2 \\
& Batchsize & 32 \\
& Activation & ReLU \\
& Optimizer & SGD \\
& Learning rate & 0.1 \\
& Momentum & 0.5 \\
& Global epochs & 50 \\
\Xhline{0.5\arrayrulewidth}
\multirow{11}{*}{\makecell{\ours, Local model\\\ours \ + Ave, Local model}} & Input dim & 93 \\
& Layers & [32,32,32] \\
& Output dim & 1 \\
& Loss & cross entropy \\
& Dropout & 0.2 \\
& Batchsize & 32 \\
& Activation & ReLU \\
& Optimizer & SGD \\
& Learning rate & 0.1 \\
& Momentum & 0.5 \\
& Local epochs & 10 \\
\Xhline{0.5\arrayrulewidth}
\multirow{11}{*}{\makecell{\ours \ + Ave, Local adversary}} & Input dim & 93 \\
& Layers & [32,32,32] \\
& Output dim & 2 \\
& Loss & cross entropy \\
& Dropout & 0.2 \\
& Batchsize & 32 \\
& Activation & ReLU \\
& Optimizer & SGD \\
& Learning rate & 0.1 \\
& Momentum & 0.5 \\
& Local epochs & 10 \\
\Xhline{0.5\arrayrulewidth}
\multirow{11}{*}{\makecell{\ours, Global model\\\ours \ + Ave, Global model}} & Input dim & 93 \\
& Layers & [32,32,32] \\
& Output dim & 2 \\
& Loss & cross entropy \\
& Dropout & 0.2 \\
& Batchsize & 32 \\
& Activation & ReLU \\
& Optimizer & SGD \\
& Learning rate & 0.1 \\
& Momentum & 0.5 \\
& Global epochs & 10 \\
\Xhline{3\arrayrulewidth}
\end{tabular}
\label{fair:config}
\end{table}

\vspace{-1mm}
\section{Discussion and Future Work}
\vspace{-1mm}

We believe that \ours\ is a general approach that offers several extensions for future work.

Firstly, combining \ours\ with existing work on compressing the number of parameters and gradient updates could further improve the efficiency of federated learning. For example, existing work in sparsifying the data and model~\cite{pmlr-v70-wang17f}, developing efficient gradient-based methods~\cite{Wang2018CooperativeSA,NIPS2017_7218}, and compressing the updates~\cite{DBLP:journals/corr/abs-1802-06058,DBLP:journals/corr/abs-1901-03040,46622} can all be applied to our local and global models. In particular, sparsifying the model through techniques such as distillation~\cite{DBLP:journals/corr/abs-1802-05668} and hashing~\cite{pmlr-v37-chenc15} could help to store the local models on devices with small memory and computational power.

Secondly, depending on the test time scenario (i.e. local test vs new test), there is a trade off between the ideal size of local models and the global model. If we know which device the test data belongs to, then having a more accurate local model would allow us to perform better prediction at test time. However, if we do not know which device the test data belongs to, it is important to use a more accurate global model to learn the true data distribution across all devices. Therefore, another step for future work would be dynamically learn the number of layers spread across the local and the global models, in a manner similar to learning dynamic computation steps in neural networks~\cite{DBLP:journals/corr/abs-1902-01046}. Different devices which contain different data distributions could use different local models which are dynamically learnt rather than hand-designed by the user. Techniques in neural architecture search could also be helpful for this purpose.

Finally, learning fair representations is of utmost importance as our machine learning systems are deployed in real-life settings such as healthcare, law, and policy-making. In addition to the adversarial training method we described in this paper, there are a variety of methods for learning fair representations that can also be incorporated into our flexible local models. For example, recent work has shown that pre-trained word and sentence representations encode and exacerbate gender, race, and religious biases~\cite{bolukbasi2016man,DBLP:journals/corr/abs-1903-10561,DBLP:journals/corr/abs-1904-03310}. Incorporating these debiasing methods for text data would be an important step towards learning \textit{fair} and \textit{unbiased} local representations in federated learning.

\end{document}